%% file: neurips_2025.tex
\documentclass{article}

\PassOptionsToPackage{numbers, compress}{natbib}
\usepackage[preprint]{neurips_2025}

\usepackage[utf8]{inputenc} 
\usepackage[T1]{fontenc}    
\usepackage{hyperref}       
\usepackage{url}            
\usepackage{booktabs}       
\usepackage{amsfonts}       
\usepackage{nicefrac}       
\usepackage{microtype}      
\usepackage{xcolor}         
\usepackage{amsthm, amsmath, amssymb}
\usepackage{mathtools}
\usepackage{comment}
\usepackage{graphicx}
\usepackage{subcaption}
\usepackage{wrapfig}
\usepackage{algorithm}
\usepackage{algorithmic}
\usepackage[capitalize,noabbrev]{cleveref}

\theoremstyle{plain}
\newtheorem{theorem}{Theorem}[section]
\newtheorem{proposition}[theorem]{Proposition}
\newtheorem{lemma}[theorem]{Lemma}

\theoremstyle{definition}
\newtheorem{definition}[theorem]{Definition}
\newtheorem{assumption}[theorem]{Assumption}
\theoremstyle{remark}
\newtheorem{remark}[theorem]{Remark}

\newcommand{\cS}{\mathcal{S}}
\newcommand{\cA}{\mathcal{A}}
\newcommand{\cN}{\mathcal{N}}
\newcommand{\cD}{\mathcal{D}}
\newcommand{\cE}{\mathcal{E}}
\newcommand{\cR}{\mathcal{R}}
\newcommand{\mR}{\mathbb{R}}
\newcommand{\mE}{\mathbb{E}}

\DeclareMathOperator{\argmax}{\operatorname*{\arg\max}}

\title{Action Dependency Graphs for Globally Optimal Coordinated Reinforcement Learning}

\author{%
  Jianglin Ding,~~ Jingcheng Tang,~~ Gangshan~Jing\thanks{Corresponding author.} 
\\  School of Automation, Chongqing University \\
  \texttt{20191921@cqu.edu.cn}\\
  \texttt{tangjingcheng@stu.cqu.edu.cn}\\
  \texttt{jinggangshan@cqu.edu.cn} 
}

\begin{document}

\maketitle

\begin{abstract}
    Action-dependent individual policies, which incorporate both environmental states and the actions of other agents in decision-making, have emerged as a promising paradigm for achieving global optimality in multi-agent reinforcement learning (MARL). 
    However, the existing literature often adopts auto-regressive action-dependent policies, where each agent's policy depends on the actions of all preceding agents. 
    This formulation incurs substantial computational complexity as the number of agents increases, thereby limiting scalability. 
    In this work, we consider a more generalized class of action-dependent policies, which do not necessarily follow the auto-regressive form.
    We propose to use the `action dependency graph (ADG)' to model the inter-agent action dependencies. Within the context of MARL problems structured by coordination graphs, 
    we prove that an action-dependent policy with a sparse ADG can achieve global optimality, provided the ADG satisfies specific conditions specified by the coordination graph. 
    Building on this theoretical foundation, we develop a tabular policy iteration algorithm with guaranteed global optimality. 
    Furthermore, we integrate our framework into several SOTA algorithms and conduct experiments in complex environments. 
    The empirical results affirm the robustness and applicability of our approach in more general scenarios, underscoring its potential for broader MARL challenges.
\end{abstract}

\section{Introduction}

\input{body/Introduction.tex}

\section{Related Work}

\input{body/Related_Work.tex}

\section{Preliminary}\label{sec:preliminary}

\input{body/Preliminary.tex}

\section{ADG with Optimality Guarantee}\label{sec:dependency}

\input{body/Dependency_Graph_with_Optimality_Guarantee.tex}

\section{Action-Dependent Policy Iteration}\label{sec:MPI}

\input{body/Multi-Agent_Policy_Iteration.tex}

\section{Practical Algorithms}\label{sec:Practical Algorithm}

\input{body/Practical_Algorithm.tex}

\section{Experiments}

\input{body/Experiments.tex}

\section{Conclusion}

    In this paper, we introduced a novel theoretical framework for MARL with action-dependent policies and CGs.
    We rigorously proved that a $G_d$-locally optimal policy attains global optimality when the ADG meets CG-defined conditions under finite state and action spaces.
    Furthermore, by embedding our theory within the SOTA MARL algorithms, we provided empirical evidence on its effectiveness in practical scenarios. 
    Recognizing that complex environments may involve unknown CGs, dynamic CGs, hypergraph CGs, we aim to explore the adaptability and potential of ADGs in these challenging settings in future research.

\nocite{*}
\bibliographystyle{unsrtnat}
\bibliography{reference}


\appendix

\section{Notations in Appendix}

\input{appendix/Notations.tex}

\section{Proof of Optimality}\label{sec:proof-optimality}

\input{appendix/Proof_of_Optimality.tex}

\section{Proof of Convergence}\label{sec:proof-convergence}

\input{appendix/Proof_of_Convergence.tex}

\section{A Class of Markov Games where CGs Exactly Exist}\label{sec:proof-preliminary}

\input{appendix/Proof_of_Preliminary.tex}

\section{Proof of the Equivalence Between Nash Equilibrium and Agent-by-Agent Optimal}

\input{appendix/Proof_of_Nash.tex}

\section{Construction of Action-Dependent Graphs}\label{sec:greedy-alg}

\input{appendix/Greedy_Algorithm}

\section{Experimental Details}\label{sec:details}

\input{appendix/Experimental_Details}

\end{document}

%% file: body/Introduction.tex
Achieving effective multi-agent reinforcement learning (MARL) in fully cooperative environments requires agents to coordinate their actions to maximize collective performance. 
Most existing MARL methods rely on independent policies \cite{zhang2021multi, oroojlooy2023review}, where each agent makes decisions based solely on its state or observation. 
Although computationally tractable and scalable, these completely decentralized policies are often suboptimal \cite{fu2022revisiting}. 
The primary limitation lies in their tendency to converge to one of many Nash equilibrium solutions \cite{ye2022towards, jing2024distributed}, which may not correspond to the globally optimal solution.

The emergence of action-dependent policies \cite{fu2022revisiting} offers a promising solution to this challenge. 
By incorporating the actions of other agents into an agent's decision-making process, 
action-dependent policies enable more effective cooperation and achieve superior performance compared to independent policies. 
We introduce the action dependency graph (ADG), a directed acyclic graph, to represent the action dependencies required for agents to make decisions.
Theoretical studies \cite{bertsekas2021multiagent, chen2023context} demonstrate that policies with auto-regressive forms, 
associated with fully dense ADGs—where replacing each directed edge with an undirected edge yields a complete graph—guarantee global optimality.
However, fully dense ADGs pose substantial scalability issues, as they require a high degree of interdependence and coordination.

Sparse ADGs, which involve fewer inter-agent dependencies, offer a more scalable alternative. 
This leads to a critical question: can action-dependent policies with sparse ADGs still guarantee global optimality?
To answer this question, we build on the framework of coordinated reinforcement learning \cite{guestrin2002coordinated}, where the cooperative relationship between agents is described by a coordination graph (CG). We find that global optimality can still be achieved using an action-dependent policy with a sparse ADG, 
provided that a specific relationship between the ADG and the CG is satisfied.

The contributions of this paper are summarized as follows.
(i). We introduce a novel theoretical framework that bridges the gap between scalability and optimality of MARL. 
To the best of our knowledge, this is the first work that seamlessly integrates coordination graphs and action-dependent policies.
(ii). We provide a theoretical demonstration of the advantages of action-dependent policies with sparse ADGs over independent policies. 
Specifically, we prove that for MARL problems characterized by sparse CGs, policies associated with sparse ADGs satisfying our proposed condition achieve global optimality. This work pioneers the analysis of policies with sparse ADGs, in contrast with prior studies \cite{ye2022towards, bertsekas2021multiagent, wang2022more}, which focus exclusively on fully dense ADGs. 
(iii). We integrate our theoretical insights with state-of-the-art algorithms, and show their effectiveness in the scenarios of coordination plymatrix games, adaptive traffic signal control, and StarCraft II.

%% file: body/Related_Work.tex
{\bfseries Independent policy.} 
The majority of the literature on MARL represents the joint policy as the Cartesian product of independent individual policies. 
Value-based methods such as IQL \cite{tan1993multi}, VDN \cite{sunehag2017value}, QMIX \cite{rashid2020monotonic}, and QTRAN \cite{son2019qtran} employ local value functions that depend only on the state or observation of each agent. 
Similarly, policy-based methods such as MADDPG \cite{lowe2017multi}, COMA \cite{foerster2018counterfactual}, MAAC \cite{iqbal2019actor}, and MAPPO \cite{yu2103surprising} directly adopt independent policies. 
These approaches often fail to achieve global optimality, as they are not able to cover all strategy modes \cite{fu2022revisiting}.

{\bfseries Coordination graph.}
Some value-based methods \cite{bohmer2020deep, li2020deep, wangcontext, castellini2021analysing} recognize that the limitation of independent policies is due to a game-theoretic pathology known as relative overgeneralization \cite{panait2006biasing}.
To mitigate this, they employ a higher-order value decomposition framework by introducing the coordination graph (CG) \cite{guestrin2002coordinated}. 
In this graph, the vertices represent agents, and the edges correspond to pairwise interactions between agents in the local value functions. 
While CGs improve cooperation by considering inter-agent dependencies, the resulting joint policy cannot be decomposed into individual policies. 
Consequently, decision-making algorithms still require intensive computation, such as Max-Plus \cite{rogers2011bounded} or Variable Elimination (VE) \cite{bertele1972nonserial}. 
When the CG is dense, these computations may become prohibitively time consuming, making the policy difficult to execute in real time.

{\bfseries Action-dependent policy.}
In contrast to independent policies, action-dependent policies \cite{wang2022more, ruan2022gcs, li2023ace, li2024backpropagation} incorporate not only the state, but also the actions of other agents into an agent's decision-making process. 
The action dependencies among agents can be represented by a directed acyclic graph, which we refer to as the action dependency graph (ADG).
In some literature, the action-dependent policy is also referred to as Bayesian policy \cite{chen2023context} or auto-regressive policy \cite{fu2022revisiting}. 
Moreover, the use of action-dependent policies can be viewed as a mechanism to leverage communications for enhancing cooperation \cite{duan2024group, zhou2023centralized}.
Some approaches \cite{ye2022towards, bertsekas2021multiagent, wen2022multi} transform a multi-agent MDP into a single-agent MDP with a sequential structure, 
enabling each agent to consider the actions of all previously decided agents during decision-making. 
This transformation ensures the convergent joint policy to be globally optimal \cite{bertsekas2021multiagent}. 
However, the fully dense ADG makes these methods computationally expensive and impractical for large-scale systems. 
For more general ADGs, existing theories can only guarantee convergence to a Nash equilibrium solution \cite{chen2023context}. 
Currently, no theoretical evidence demonstrates the superiority of action-dependent policies with sparse dependency graphs over independent policies.

%% file: body/Preliminary.tex
In this paper, we formulate the cooperative multi-agent reinforcement learning problem as a Markov game, 
defined by the tuple $\langle \mathcal{N}, \mathcal{S}, \mathcal{A}, P, r, \gamma \rangle$, 
where $\mathcal{N} = \{1, \dots, n\}$ denotes the set of agents, $\mathcal{S}$ represents the finite state space, 
$\mathcal{A} = \prod_{i=1}^n \mathcal{A}_i$ is the Cartesian product of the agents’ finite action spaces, 
$P: \mathcal{S} \times \mathcal{A} \times \mathcal{S} \to [0,1]$ specifies the Markovian transition model, 
$r: \mathcal{S} \times \mathcal{A} \to \mathbb{R}$ defines the reward function, and $\gamma \in [0,1)$ is the discount factor.

A deterministic joint policy $\pi: \cS\to\cA$ maps a state to a joint action,
while a stochastic joint policy $\pi: \cS\to\Delta(\cA)$ maps a state space to a probability distribution over the joint action space.
Typically, a joint policy is represented as the product of \emph{independent policies} $\pi_i: \cS\to\Delta(\cA_i)$, such that $\pi(a|s) = \prod_{i=1}^n\pi_i(a_i|s)$.
The state value function induced by a joint policy $\pi$ is defined as:
\begin{equation}
    V^{\pi}(s) := \mE_{a^t\sim\pi} \left[ \sum_{t=0}^{\infty} \gamma^t r(s^t, a^t) \middle| s^0 = s \right],
\end{equation}
and the corresponding state-action value function is defined as:
\begin{equation}
    Q^{\pi}(s,a) := \mE_{a^t\sim\pi} \left[ \sum_{t=0}^{\infty} \gamma^t r(s^t, a^t) \middle| s^0 = s, a^0 = a \right].
\end{equation}
Given any $V \in \cR(\cS)$, where $\cR(\cS)$ denotes the set
of all real-valued functions $J: \cS\to\mR$, we define:
\begin{equation}
    Q^V(s,a) = r(s,a) + \gamma\sum_{s'\in\cS}P(s'|s,a)V(s').
\end{equation}

The Bellman operator $T_{\pi}: \mathcal{R}(\mathcal{S}) \to \mathcal{R}(\mathcal{S})$ and the Bellman optimal operator $T: \mathcal{R}(\mathcal{S}) \to \mathcal{R}(\mathcal{S})$ are defined as:
\begin{equation}
    T_{\pi}V(s) = \sum_{a\in\cA}\pi(a|s)Q^V(s,a), \quad TV(s) = \max_{a\in\cA}Q^V(s,a)
\end{equation}
The value function $V^{\pi}$ satisfies the Bellman equation $T_{\pi}V^\pi = V^\pi$, while the optimal value function $V^*$ satisfies the Bellman optimal equation $TV^* = V^*$.



\subsection{Coordination Graph}


\begin{wrapfigure}{r}{0.45\linewidth}
    \vskip -0.3in
    \centerline{\includegraphics[width=\linewidth]{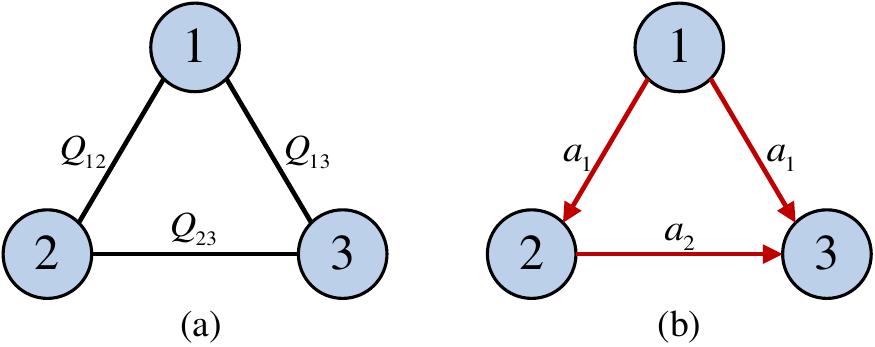}}
    \caption{A coordination graph (a) and an action dependency graph (b).}\label{fig:cg-adg}
    \vskip -0.2in
\end{wrapfigure}

In many practical scenarios such as sensor networks \cite{zhang2011coordinated}, wind farms \cite{bargiacchi2018learning}, mobile networks \cite{bouton2021coordinated}, etc., 
the value function $Q$ can be approximated as the sum of local value functions, each depending on the states and actions of a subset of agents.
A widely used approach to representing this decomposition is the use of the \emph{coordination graph} (CG) \cite{guestrin2002coordinated}, which captures the pairwise coordination relationships between agents.
Formally, we define a CG as follows.

\begin{definition}[Coordination Graph]
    An undirected graph\footnote{
    In this paper,  the vertices and edges of the graph are represented by agent indices and index pairs. 
    }
    $G_{c} = (\cN, \cE_{c})$ is a CG of a function $Q: \cS\times\cA\to\mR$, if  
    there exists a local value function $Q_{ij}:\cS\times\cA_i\times\cA_j\to\mR$ for every edge $(i,j)\in\cE_{c}$, and a local value function $Q_{i}:\cS\times\cA_i\to\mR$ for every vertex $i\in\cN$, such that
    \begin{equation}
        Q(s,a) = \sum_{i\in\cN}Q_{i}(s,a_i) + \sum_{(i,j)\in\cE_{c}}Q_{ij}(s,a_i,a_j).
    \end{equation}
\end{definition}
\begin{remark}
    If $G_c$ is a subgraph of $G'_c$, and $G_c$ is a CG of $Q$, then $G'_c$ is also a CG of $Q$.
    Therefore, multiple CGs may correspond to the same value function $Q$.
\end{remark}
Without loss of generality, we assume that $G_{c}$ is connected; otherwise, the problem can be decomposed into independent subproblems depending on the connected components of $G_c$.
In a connected graph, each vertex is involved in an edge,
allowing the local value functions associated with vertices to be merged into those local value functions associated with edges, yielding:
\begin{equation}\label{Qsum}
    Q(s,a) = \sum_{(i,j)\in\cE_{c}}Q_{ij}(s,a_i,a_j).
\end{equation}
\cref{fig:cg-adg} (a) shows a CG where the function $Q$ can be decomposed as:
\begin{equation}
    Q(s,a) = Q_{12}(s,a_1,a_2) + Q_{13}(s,a_1,a_3) + Q_{23}(s,a_2,a_3).
\end{equation}

We further characterize a class of Markov games where the equation \eqref{Qsum} holds exactly, providing a natural generalization of the concept of decomposable games \cite{dou2022understanding}.
Due to the space limitation, the details is provided in the appendix.

\subsection{Action Dependency Graph}

In MARL, a joint policy is typically represented as the product of independent policies, each depending only on the current state.
However, not every distribution over the joint action space can be represented as such a product. 
Given a distribution $\pi$ over the joint action space, the chain rule of conditional probability states:
\begin{equation}
    \pi(a|s) = \prod_{i=1}^n\pi_i(a_i|s,a_1,\dots,a_{i-1}).
\end{equation}
This formulation gives rise to a broader class of individual policies, termed \emph{action-dependent policies}, 
which are not necessarily conditionally independent but may instead depend on the actions of other agents.
To formalize the action-dependent policy, we associate a joint policy $\pi$ with a directed acyclic graph $G_{d} = (\cN, \cE_{d})$, called the \emph{action dependency graph} (ADG).
\begin{definition}[Action Dependency Graph]
A directed acyclic graph $G_{d} = (\cN, \cE_{d})$ is said to be the ADG of a policy $\pi$ if the policy of each agent $i$ depends on the actions of agents in $N_{d}(i)=\{j\in\mathcal{N}:(j,i)\in\mathcal{E}_d\}$. More specifically, the joint policy $\pi$ is expressed as:
    \begin{equation}
        \pi(a|s) = \prod_{i=1}^n\pi_i(a_i|s,a_{N_{d}(i)}),
    \end{equation}
where $\pi_i: \cS\times\prod_{j\in N_{d}(i)}\cA_j\to\Delta(\cA_i)$ is the policy of agent $i$,
and $a_{N_d(i)}=(a_j)_{j\in N_d(i)}$\footnote{Throughout this paper, we use a vector with a set as its subscript to denote the components of the vector indexed by elements in that set.} denotes the subvector of $a$ indexed by $N_d(i)$.
\end{definition}
\cref{fig:cg-adg} (b) shows the ADG of a joint policy $\pi$ with the following expression:
\begin{equation}
    \pi(a|s) = \pi_1(a_1|s)\pi_2(a_2|s, a_1)\pi_3(a_3|s, a_1, a_2).
\end{equation}
The acyclic nature of the ADG ensures that action dependencies do not form loops, which would otherwise lead to decision-making deadlocks, rendering polices infeasible. 
Every directed acyclic graph admits at least one topological sort, which can determine the order in which agents generate their actions.
Without loss of generality, we assume that the indices of agents in $\cN$ follow a topological sort.
This implies that $j < i$ for all $j \in N_{d}(i)$ and all $i\in\cN$.

Throughout this paper, we consider a Markov game structured by a CG. For analytical simplicity, we assume full observability, which means that each agent has access to the global state $s$. Nonetheless, our theoretical results can be integrated into practical algorithms, so as to adapt to scenarios with partial observability, which will be shown in \cref{sec:Practical Algorithm}.

%% file: body/Dependency_Graph_with_Optimality_Guarantee.tex
While many methods designed for independent policies cannot guarantee convergence to a globally optimal policy, 
theoretical results do exist that ensure convergence to a Nash equilibrium policy \cite{zhang2022global, kuba2022trust}, which satisfies the following condition:
\begin{equation}\label{eq:Nash Equilibrium}
    V^{\pi}(s) = \max_{\pi'_i} V^{\pi'_i, \pi_{-i}}(s), \forall s\in\cS ,i\in\cN.
\end{equation}
A Nash equilibrium policy is equivalent to the agent-by-agent optimal policy introduced in \cite{bertsekas2021multiagent} (see proof in \cref{prop:nash}).
Specifically, a policy $\pi$ is agent-by-agent optimal if:
\begin{equation}\label{eq:agent-by-agent}
    V^{\pi}(s) = \max_{\pi'_i} \mE_{a_i\sim\pi'_{i}, a_{-i}\sim\pi_{-i}} \left[ Q^{\pi}(s,a) \right], \forall s\in\cS ,i\in\cN.
\end{equation}
Here, $-S$ denotes the complement of set $S \subseteq \cN$, and $-i$ is the shorthand for $- \{i\}$. 
The joint policy of agents in $S$ is denoted by $\pi_{S}$.
we also write joint policy $\pi$ as $(\pi_{S}, \pi_{-S})$.
For brevity, we use $\mE_{\pi} \left[ \cdot \right]$ as the shorthand for $\mE_{a \sim\pi} \left[ \cdot \right]$, omitting the action when the context is clear.   

Unlike independent policies, an action-dependent policy takes the actions of other agents (determined by $G_d$) as an additional part of the state representation.
This enables the algorithms to achieve stronger solutions than agent-by-agent optimal policy,
which we refer to as $G_d$-locally optimal policy.
\begin{definition} 
    A joint policy $\pi$ is said to be $G_d$-locally optimal if for all $(s, a_{N_{d}(i)})\in\cS\times\prod_{j\in N_{d}(i)}\cA_j$ and $i\in\cN$, we have:
    \begin{equation}\label{eq:gd-locally}
        \mE_{\pi_{-N_{d}(i)}} \left[ Q^{\pi}(s,a) \right] = \max_{\pi'_{i}} \mE_{\pi'_{i}, \pi_{-N_{d}[i]}} \left[ Q^{\pi}(s,a) \right], 
    \end{equation}
    where $N_{d}[i] := N_{d}(i)\cup \{i\}$ denotes closed in-degree neighborhood of $i$ in $G_d$.
\end{definition}

\begin{remark}
    When $G_d$ is an empty graph, the definition of $G_d$-locally optimal is equivalent to that of agent-by-agent optimal. 
    As the number of edges in $G_d$ increases, the condition \eqref{eq:gd-locally} becomes more restrictive. 
    This results in fewer agent-by-agent optimal policies meeting the $G_d$-locally optimal criterion. 
    In extreme, when $G_d$ is a fully dense directed acyclic graph, the $G_d$-locally optimal policy aligns with the globally optimal policy. 
    When the discussion is restricted to deterministic policy class, a $G_d$-locally optimal policy is always agent-by-agent optimal, though the converse does not necessarily hold. 
    In this context, the class of $G_d$-locally optimal policies can be rigorously characterized as a subset of the agent-by-agent optimal policy class. 
\end{remark}

\subsection{Coordination Polymatrix Game}\label{sec:Polymatrix}
To illustrate the suboptimality of Nash equilibrium and how action-dependent policies can address this issue, we consider an example of a \emph{coordination polymatrix game} \cite{cai2011minmax}. 
A coordination polymatrix game can be viewed as a single-step decision problem, defined by the Markov game tuple $\langle \cN, \cS, \cA, P, r, \gamma \rangle$,
where $\cS=\emptyset$ and $\gamma=0$. Additionally, the game includes an undirected graph $G_c=(\cN, \cE_c)$ and a set of payoff functions $\{r_{ij}\}_{(i,j)\in\cE_c}$,
which determine the global reward $r(a) = \sum_{(i,j)\in\cE_{c}}r_{ij}(a_i,a_j)$.
In this context, $r$ is equivalent to the (state)-action value function $Q:\cA\to\mR$, and $G_c$ is the CG of $Q$. \cref{fig:polymatrix} depicts a polymatrix game with three agents, each having two actions, $\cA_i=\{0,1\}, i=1,2,3$. 
The payoff matrices specify the rewards for each pair of agents. For example, if agents $1$ and $2$ both choose action $0$, they receive a payoff of $1$.


\begin{wrapfigure}{r}{0.5\linewidth}
    \centerline{\includegraphics[width=\linewidth]{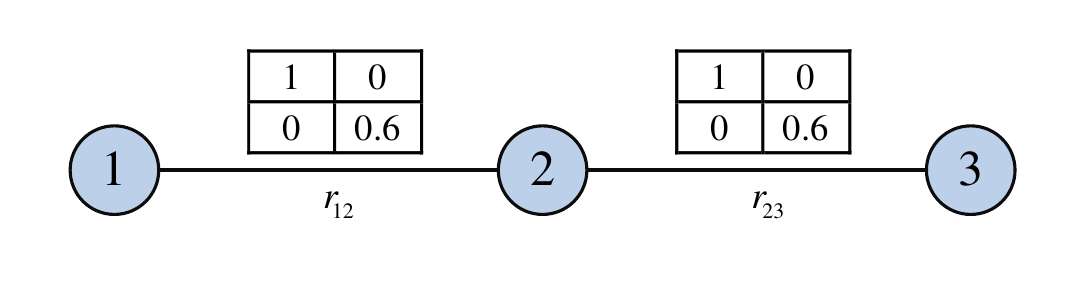}}
    \vskip -0.1in
    \caption{A polymatrix game on a line CG.}\label{fig:polymatrix}
\end{wrapfigure}

When using deterministic independent policies, the joint policies $\pi=(1,1,1)$ and $\pi=(0,0,0)$ are both agent-by-agent optimal.
However, only $\pi=(0,0,0)$ is globally optimal. Although $\pi=(1,1,1)$ is suboptimal,
agents lack sufficient motivation to unilaterally change their actions, highlighting the limitations of independent policies in ensuring optimal cooperation.

Now consider deterministic action-dependent policies associated with an ADG $G_d$, where $\cE_d=\{(1,2),(2,3)\}$.
When $\pi_1=1, \pi_2(0)=0, \pi_2(1)=1, \pi_3(0)=0, \pi_3(1)=1$, the joint policy remains $\pi=(1,1,1)$. If agent $1$ switches to action $0$, agents $2$ and $3$ will also switch to action $0$,
resulting in $r(0,0,0) = 2$, which exceeds $r(1,1,1) = 1.2$.
This incentivizes agent $1$ to choose action $0$, leading to the globally optimal policy, which is also the only $G_d$-locally optimal policy in this case.
However, not all ADGs guarantee that a $G_d$-locally optimal policy is globally optimal.
For instance, if $G_d$ has an edge set $\cE_d=\{(1,2), (3,2)\}$. When $\pi_1=1, \pi_3 =1, \pi_2(0,0)=\pi_2(1,0)=\pi_2(0,1)=0, \pi_2(1,1)=1$, the joint policy $\pi=(1,1,1)$ is $G_d$-locally optimal but not globally optimal.

\subsection{Optimality Guarantee}

There are some common criteria to determine whether a local optimum is also globally optimal, 
such as convexity of continuous functions or complete unimodality of discrete functions \cite{bauso2012team}.
However, verifying these conditions often requires prior knowledge of the problem model, which is typically impractical.
In this section, we introduce a novel condition for ensuring the global optimality of $G_d$-locally optimal policies, which is associated with the relationship between CG and ADG.
The verification of this condition requires the network structure only.

Before presenting the results, we introduce some notations. 
Let $N_c(i)$ denote the neighbors of $i$ in $G_c$.
When $i$ is replaced with a set $S$, $N_c(S)$ represents the neighbors of all agents in $S$, i.e., $N_c(S)=\left(\cup_{i\in S}N_c(i)\right) \setminus S$.
Let $i^{[+]} = \{i, i+1, \dots, n\}$ represent the set of agents with indices not less than $i$.
\begin{theorem}[Optimality of ADG, Proof in \cref{sec:proof-optimality}]\label{thm:optimal}
    Let policy $\pi$ associated with ADG $G_{d} = (\cN, \cE_{d})$ be $G_d$-locally optimal. 
    Let $G_{c} = (\cN, \cE_{c})$ be a CG of state-action value $Q^{\pi}$.
    If:
    \begin{equation}\label{eq:graph-condition}
        N_{d}(i) = N_c(i^{[+]}), \forall i\in \cN,
    \end{equation}
    then $\pi$ is globally optimal, meaning $V^{\pi}=V^{*}$.
\end{theorem}

\begin{wrapfigure}{r}{0.55\linewidth}
    \vskip -0.2in
    \centerline{\includegraphics[width=\linewidth]{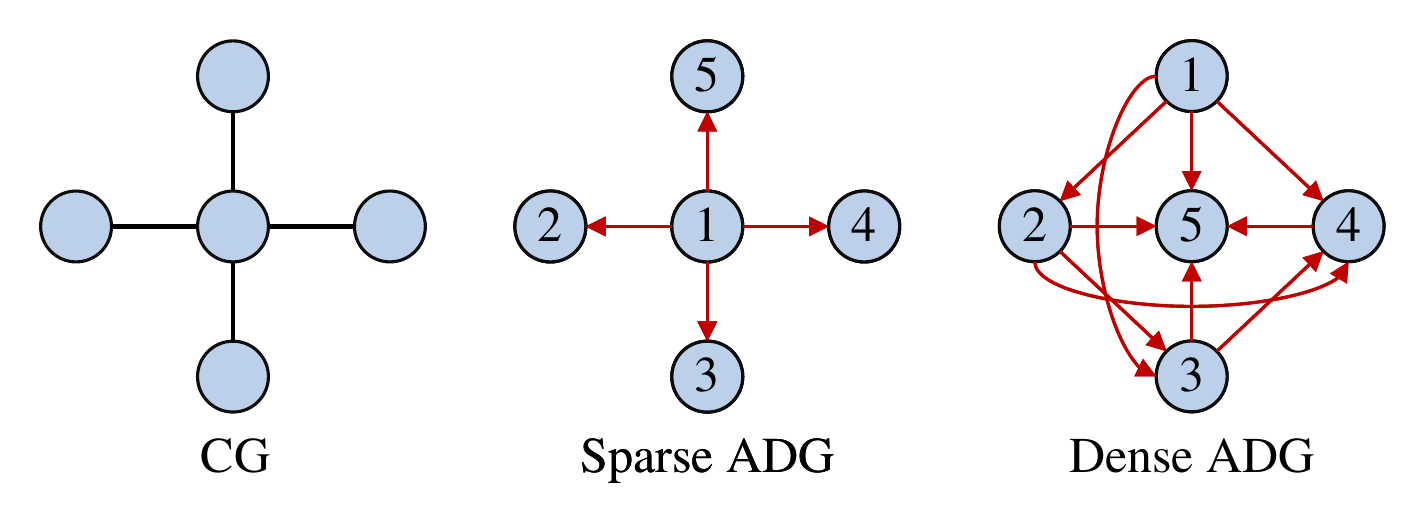}}
    \caption{
        Different index orders of agents result in different sparsity of the ADG.
    }\label{fig:index-order}
    \vskip -0.1in
\end{wrapfigure}

This theorem indicates that the ADG can be designed based on the CG, thereby ensuring that a $G_d$-locally optimal policy is also globally optimal.

\begin{remark} 
    \cref{thm:optimal} contains results in the literature as special cases. In particular, when both $G_c$ and $G_d$ are empty graphs, 
    the $Q$-function can be decomposed in the form of VDN. Consequently, any Nash equilibrium policy is globally optimal, aligning with the findings in \cite{dou2022understanding}.
    On the other hand, when $G_c$ is a complete graph and $G_d$ be a fully dense graph with edge set $\mathcal{E}_d = \{(i,j)\in\mathcal{N}\times\mathcal{N}: i < j\}$, the condition $N_{d}(i) = N_c(i^{[+]})$ is also satisfied.
    Thus, a $G_d$-locally optimal policy is guaranteed to be globally optimal for any CG, since every CG is a subgraph of a complete graph, aligning with the findings in \cite{chen2023context}.
\end{remark}

While \cref{thm:optimal} requires the availability of a CG, which may seem restrictive, 
this is reasonable in many practical scenarios where the CG can be predefined. 
Moreover, even in domains without an apparent network structure, an approximate CG can still be learned \cite{li2020deep, wangcontext}.

\subsection{Finding Sparse ADGs}


\begin{wrapfigure}{r}{0.5\linewidth}
    \vskip -0.3in
    \centerline{\includegraphics[width=\linewidth]{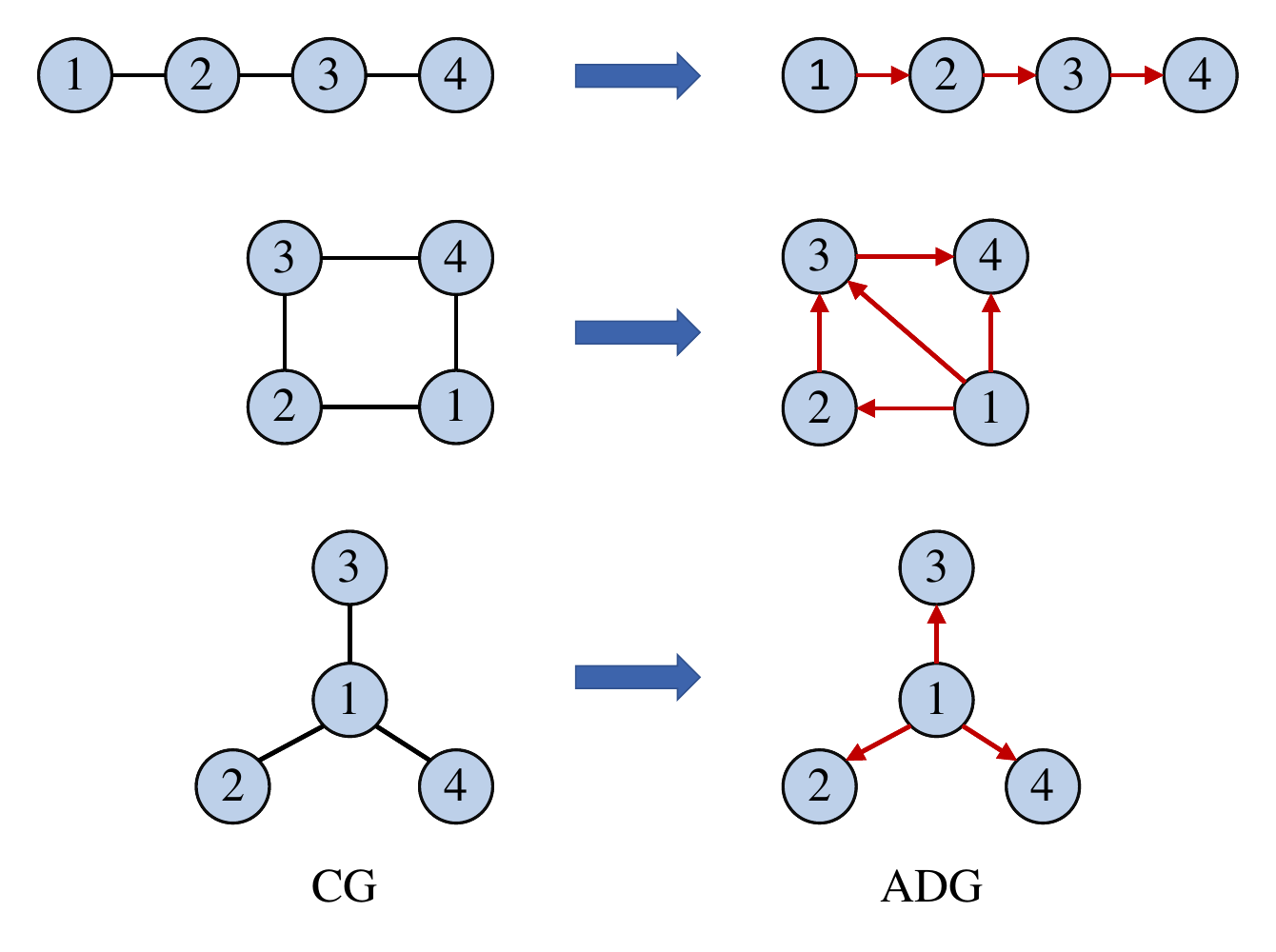}}
    \vskip -0.1in
    \caption{
        ADGs generated by Algorithm \ref{alg:greedy} for CG topologies: line, ring, and star. 
    }\label{fig:greedy-example}
    \vskip -0.15in
\end{wrapfigure}

Although a fully dense ADG $G_d$ ensures that any $G_d$-locally optimal policy achieves global optimality, training such policies can be computationally costly, limiting scalability. Thus, for sparse CGs, identifying sparse ADGs that satisfy \eqref{eq:graph-condition} is essential to enhance efficiency.

If the indices of the agents are predetermined, the ADG can be uniquely determined by the condition \eqref{eq:graph-condition}. 
However, the choice of index order (topological sort) significantly affects the sparsity of the ADG, as illustrated in \cref{fig:index-order}.
Determining the optimal index order is analogous to finding the optimal elimination order in VE algorithm, a problem known to be NP-complete \cite{kok2006collaborative}. 
Despite this complexity, practical heuristics can be employed, such as the greedy algorithm described in \cref{sec:greedy-alg} \cref{alg:greedy}. 
The ADGs generated by the greedy algorithm for some simple network topologies are illustrated in \cref{fig:greedy-example}.

%% file: body/Multi-Agent_Policy_Iteration.tex
In this section, we introduce a policy iteration algorithm for MARL in a tabular setting.
This algorithm elucidates the advantage of employing action-dependent policies, 
enabling convergence to a $G_d$-locally optimal policy rather than merely an agent-by-agent optimal policy. 
Under the conditions established in \cref{thm:optimal}, it can be further concluded that our proposed policy iteration algorithm achieves convergence to the globally optimal solution.

Our approach extends the multi-agent policy iteration (MPI) framework proposed in \cite{bertsekas2021multiagent}, which decomposes the joint policy iteration step of standard policy iteration (PI) \cite{sutton2018reinforcement} into sequential updates of individual agents policies, thereby mitigating the computational complexity of PI. However, MPI only ensures convergence to an agent-by-agent optimal policy, which is frequently suboptimal.
In order to ensure the convergence of the joint policy to a $G_d$-locally optimal policy, we propose the \cref{alg:ad-mpi},
which integrates deterministic action-dependent policies into the MPI framework.

\begin{algorithm}
    \caption{Action-Dependent Multi-Agent Policy Iteration}\label{alg:ad-mpi}
    \begin{algorithmic}
        \STATE Initialize policy $\pi^0_i, i\in\mathcal{N}$ associated with the ADG $G_d$
        \FOR{$k=0,1,\ldots$}
        \STATE {\itshape // Policy Evaluation}
        \STATE Solve for $V^{\pi^k}$ using the Bellman equation $V = T_{\pi^k}V$
        \STATE Calculate $Q^{\pi^{k}}$ using $V^{\pi^k}$
        \STATE {\itshape // Policy Iteration}
        \FOR{$i=1,2,\ldots,n$}
        \STATE Denote the joint policy of $(\pi^{k+1}_1, \ldots, \pi^{k+1}_{i-1}, \pi^{k}_{i} \ldots, \pi^{k}_{n})$ by $\pi^{k,i}$
        \STATE Update policy: $\pi^{k+1}_{i}(s, a_{N_{d}(i)})\leftarrow \underset{a_{i}\in\cA_{i}}{\argmax}Q^{\pi^{k}}(s,a_{N_{d}[i]}, \pi^{k,i}_{-N_{d}[i]}(s, a_{N_{d}[i]})), \forall s, a_{N_{d}(i)}$
        \ENDFOR
        \ENDFOR
    \end{algorithmic}
\end{algorithm}

To prevent non-convergence caused by alternating policies with identical value functions, we introduce the \cref{asp:stopping-criterion}. In practical implementations, this assumption can be relaxed by checking whether $\pi^{k}_{i}$ already achieves the maximum before updating $\pi^{k+1}_{i}$.
If $\pi^{k}_{i}$ achieves the maximum, then $\pi^{k+1}_{i}$ is set to $\pi^{k}_{i}$ to ensure stability. 
The convergence of \cref{alg:ad-mpi} is established in \cref{thm:convergence}.

\begin{assumption}\label{asp:stopping-criterion}
    In \cref{alg:ad-mpi}, the set $\underset{a_{i}\in\cA_{i}}{\argmax}Q^{\pi^{k}}\left(s,a_{N_{d}[i]}, \pi^{k,i}_{-N_{d}[i]}(s, a_{N_{d}[i]})\right)$
    is a singleton set for all $k=1,2,\ldots$ and $i\in\mathcal{N}$.
\end{assumption}

\begin{theorem}[Convergence of \cref{alg:ad-mpi}, proof in \cref{sec:proof-convergence}]\label{thm:convergence}
    Assume $G_{c} = (\cN, \cE_{c})$ is a CG of state-action value $Q^{\pi}$ for all policies $\pi$, and \cref{asp:stopping-criterion} holds. 
    Let $\{\pi^k\}$ be a sequence of joint policies generated by \cref{alg:ad-mpi}, associated with an ADG $G_{d} = (\cN, \cE_{d})$.
    If the condition \eqref{eq:graph-condition} is satisfied,
    then $\{\pi^k\}$ converges to a globally optimal policy in finite terms.
\end{theorem}

\begin{remark}
\cref{thm:convergence} assumes that $G_c$ serves as the CG of $Q^{\pi}$ across all policies. 
However, this condition can be relaxed to require only $G_c$ to act as the CG of the value function of the converged policy. 
The detailed proof of this relaxation is provided in \cref{thm:convergence-original}.
Additionally, in \cref{alg:ad-mpi}, while the convergence of all individual policies ensures that the resulting joint policy is $G_d$-locally optimal,
the convergence is not always guaranteed if \eqref{eq:graph-condition} is not satisfied.
A similar issue has been encountered in \cite{chen2023context}, and we leave it as our future research topic.
\end{remark}

%% file: body/Practical_Algorithm.tex
In this section, we introduce how to integrate the ADG into existing SOTA algorithms. 
Unlike prior CG-based approaches \cite{bohmer2020deep, li2020deep}, which face problems in representing individual policies and rely on complex techniques such Max-Plus to compute optimal Q values, 
action-dependent policies exploit the structure of CGs while retaining the flexibility of independent policies. 
This facilitates seamless and versatile integration with a broad range of SOTA algorithms, rendering our theoretical framework applicable to general scenarios encompassing continuous state-action spaces and partial observability.

\subsection{Policy-Based Methods}

Policy-based methods optimize parameterized policies directly.
To leverage our framework, we transform the independent policy $\pi_{\theta_i}(a_i|s)$ into an action-dependent form $\pi_{\theta_i}(a_i|s, a_{N_{d}(i)})$.
This necessitates a corresponding adjustment to the optimization objective to accommodate the action-dependent policy.
For instance, consider MAPPO \cite{yu2103surprising}, where the original objective is:
\begin{equation}
    \mathcal{L}(\theta)=\mathbb{E}_{s\sim\cD,a\sim\pi_{\theta_{\mathrm{old}}}}
    \Bigg[\sum_{i=1}^{n}\min\left(r_{\theta_i}(a_i, s) A_{\pi_{\theta_{\mathrm{old}}}}(s,a), \right.\nonumber
    \left. \mathrm{clip}(r_{\theta_i}(a_i, s),1\pm\varepsilon)A_{\pi_{\theta_{\mathrm{old}}}}(s,a)\right)\Bigg],
\end{equation}
where $r_{\theta_i}(a_i, s) = \frac{\pi_{\theta_i}(a_i|s)}{\pi_{\theta_{i,\mathrm{old}}}(a_i|s)}$, $\cD$ denotes the distribution of replay buffer.
To adapt this objective for action-dependent policies, we replace $r_{\theta_i}(a_i, s)$ with
$r_{\theta_i}(a_i, s, a_{N_{d}(i)}) = \frac{\pi_{\theta_i}(a_i|s, a_{N_{d}(i)})}{\pi_{\theta_{i,\mathrm{old}}}(a_i|s, a_{N_{d}(i)})}$.



\subsection{Value-Based Methods}

Value-based methods implicitly represent policies through individual value functions $Q_i(s, a_i)$. 
To incorporate ADGs, we reformulate these into an action-dependent form $Q_i(s, a_i, a_{N_{d}(i)})$, enabling the derivation of corresponding action-dependent policies.

For example, 
in QMIX \cite{rashid2020monotonic}, we enforce a monotonicity constraint on the mixing network, ensuring:
\begin{equation}
    \frac{\partial Q(s, a)}{\partial Q_i(s, a_{N_{d}(i)}, a_i)} \geq 0, \quad \forall i\in\cN, a\in\cA, s\in\cS.
\end{equation}
To obtain the greedy action $\arg\max_a Q(s, a)$, we compute $\arg\max_{a_i} Q_i(s, a_i, a_{N_{d}(i)})$ sequentially for each agent $i = 1, \ldots, n$,
leveraging the structural dependencies encoded by the ADG.

%% file: body/Experiments.tex
To validate the theoretical results, we conduct an experimental evaluation of Algorithm 2 on coordination polymatrix games across diverse topological configurations.
We will compare the performance corresponding to three types of ADGs, namely sparse ADG (satisfying condition \eqref{eq:graph-condition}), fully dense ADG and empty ADG.
To further explore the practical applicability of our method, we integrate ADGs into practical MARL algorithms,
and evaluate their performance in more complex scenarios, including adaptive traffic signal control (ATSC) and StarCraft II. 
Detailed experimental protocols and hyperparameter settings are provided in \cref{sec:details}.

\subsection{Coordination Polymatrix Games}

\begin{figure}[htb]
    \vskip -0.1in
    \centerline{\includegraphics[width=\linewidth]{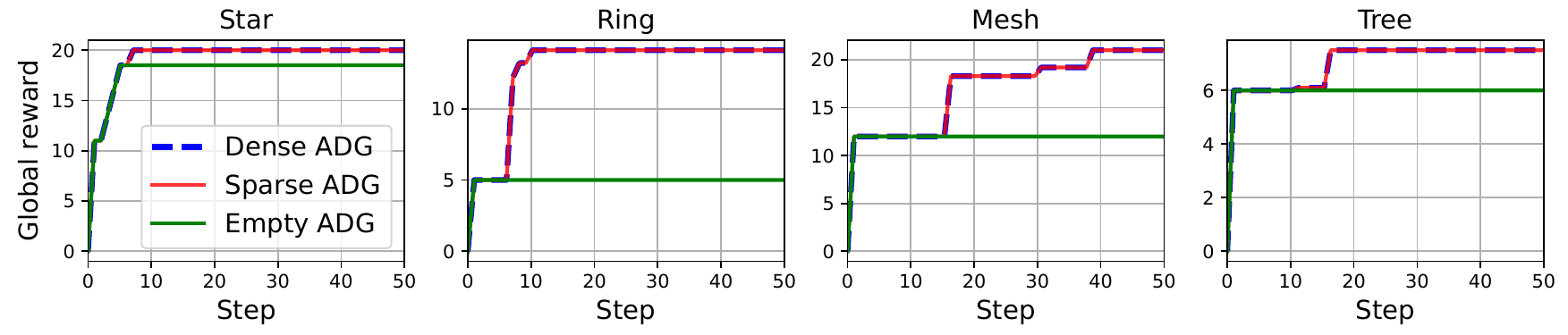}}
    \caption{Results of coordination polymatrix game.}\label{fig:ADGPL}
    \vskip -0.1in
\end{figure}

The experimental setup for the coordination polymatrix game is detailed in \cref{sec:Polymatrix}. 
We implement Algorithm 2 and evaluate its performance under various topological structures: star, ring, tree, and mesh. 
The specific CG topologies, corresponding sparse ADGs, and payoff matrices are elaborated in \cref{sec:details}.
The payoff matrices are meticulously designed to induce multiple Nash equilibria, creating a challenging environment that tests the ability to achieve optimal cooperative outcomes.

\cref{fig:ADGPL} demonstrates that both the sparse ADGs and the dense ADGs lead to globally optimal outcomes, while policies with empty ADGs frequently converge to suboptimal Nash equilibria. In fact, each learning step of the dense ADG takes more time than the sparse ADG since redundant action-dependencies are considered. These results underscore the critical role of a well-designed ADG in achieving global optimality with low computation costs.

\subsection{ATSC}

\begin{wrapfigure}{r}{0.28\linewidth}
    \vskip -0.3in
    \centering
    \includegraphics[width=0.9\linewidth]{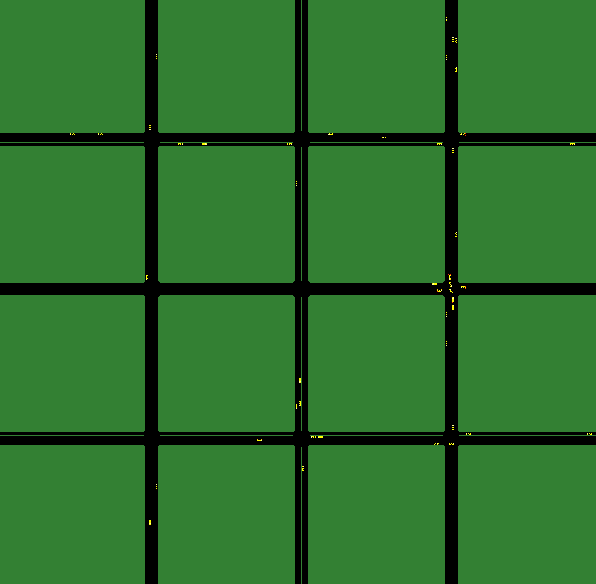}
    \caption{3x3 road network.}
    \vskip -0.2in
\end{wrapfigure}

ATSC is a benchmark with a clear cooperation structure. We perform ATSC experiments using the Simulation of Urban Mobility (SUMO) platform \cite{SUMO2018}, 
with the objective of optimizing traffic signal operations to enhance vehicular flow.
Operating under partial observability, each of nine traffic signals accesses local intersection data in a 3x3 road network provided by SUMO-RL \cite{sumorl}, with the reward defined as the negative aggregate pressure (difference between entering and exiting vehicles).
The CG is constructed based on adjacency relationships between intersections, and the sparse ADG is derived using \cref{alg:greedy}. 

\begin{wrapfigure}{l}{0.6\linewidth}
    \vskip -0.1in
    \centering
    \includegraphics[width=\linewidth]{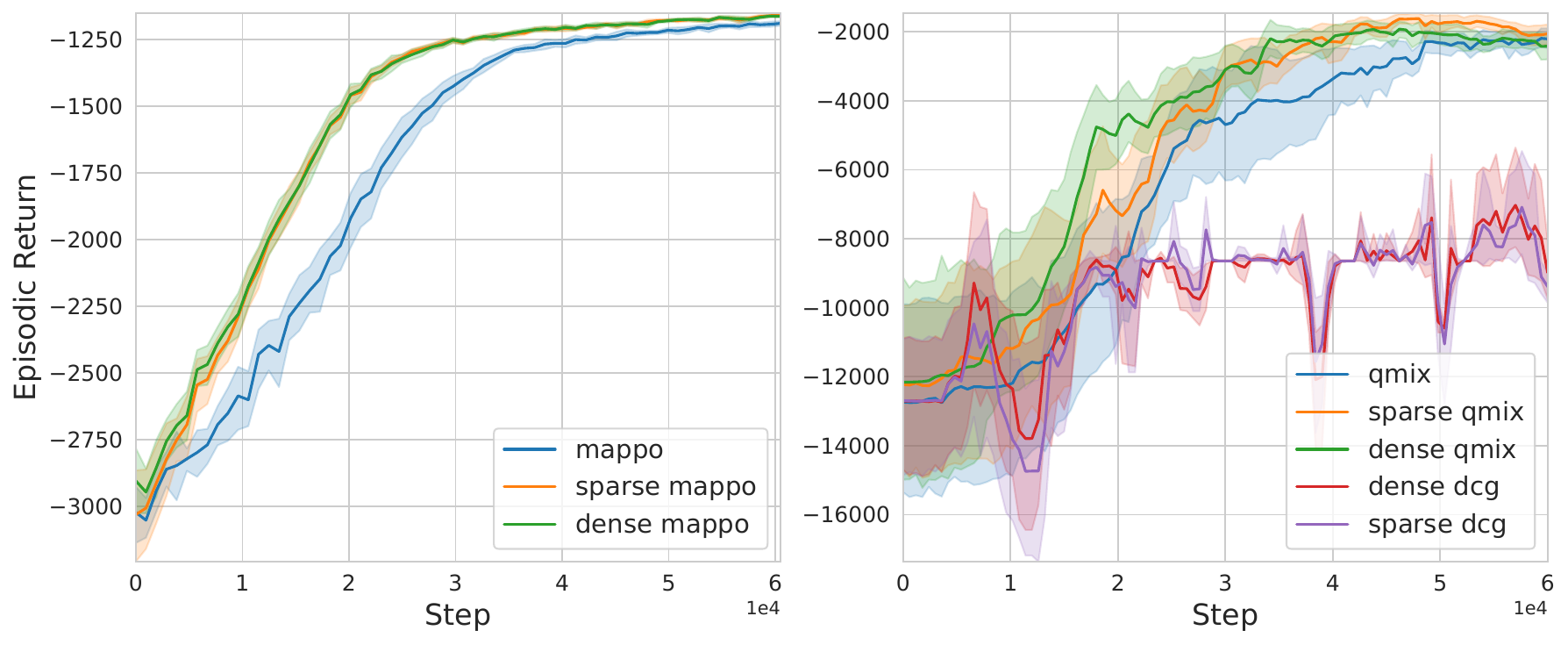}
    \caption{Results of ATSC.}\label{fig:qmix-mappo-sumo}
    \vskip -0.1in
\end{wrapfigure}

We augment the QMIX and MAPPO algorithms with ADG and compare their performance with the Deep Coordination Graph (DCG) algorithm \cite{bohmer2020deep} across ten independent experiments with distinct random seeds, as reported in \cref{fig:qmix-mappo-sumo}. 
The CG for dense DCG is a complete graph, whereas sparse DCG employs the same CG as sparse ADG. 
Both sparse and dense ADGs exhibit a marginal performance advantage over empty ADGs, coupled with faster convergence. 
The limited performance gap may be attributed to the partial observability setting, which constrains agents' ability to fully exploit action dependencies due to the restricted observational scope. 
Nevertheless, the consistent convergence behavior of sparse and dense ADGs corroborates the findings from the polymatrix game experiments. 
Conversely, DCG exhibits hyperparameter sensitivity and convergence difficulties under a restricted observational scope.


\subsection{StarCraft II}

\begin{wrapfigure}{r}{0.3\linewidth}
    \vskip -0.2in
    \centering
    \includegraphics[width=\linewidth]{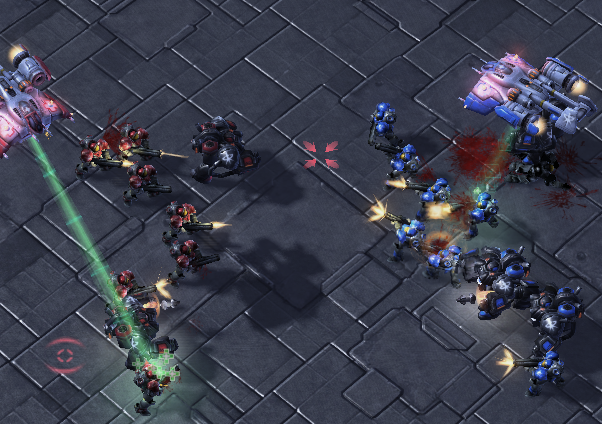}
    \caption{SMAC MMM2}
    \vskip -0.1in
\end{wrapfigure}

We adopt the SMAC \cite{samvelyan19smac} benchmark to evaluate the performance of our method in complex environments.
Recognizing the difficulty of inferring fixed CGs in StarCraft II, we introduce an artificially constructed CG and derive the corresponding ADG, as illustrated in \cref{fig:cg-adg} in appendix. 
We train the QMIX algorithm augmented with ADG and compare its performance with the DCG algorithm, which leverages the same artificial CG. 
This comparison elucidates the relative strengths of ADG-based methods versus existing CG-based methods. We conduct ten experiments with distinct seeds on the super hard map MMM2, with results reported in \cref{fig:sc2}. 

\begin{wrapfigure}{l}{0.35\linewidth}
    \vskip -0.1in
    \centering
    \includegraphics[width=\linewidth]{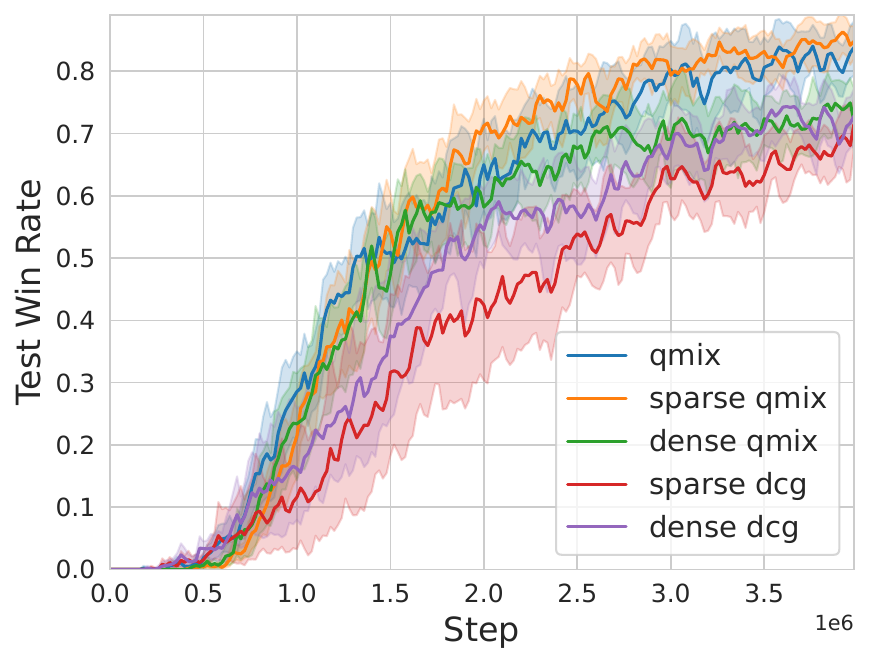}
    \vskip -0.1in
    \caption{Results of SMAC.}\label{fig:sc2}
    \vskip -0.1in
\end{wrapfigure}

The results reveal that QMIX with a sparse ADG outperforms the baseline QMIX algorithm, whereas QMIX with a dense ADG exhibits inferior performance. 
This suggests that dense ADG policies may require a greater volume of samples to achieve comparable performance. For DCG, the dense variant (using a complete CG) performs comparably to dense ADG, whereas the sparse variant (using the same CG as sparse ADG) underperforms. 
The suboptimal performance of sparse DCG may be attributable to the max-plus algorithm, which suffers from reduced accuracy as the number of CG edges decreases. In contrast, our ADG-based method circumvents reliance on the max-plus algorithm, mitigating this limitation.



%% file: appendix/Notations.tex
For clarity, we define the following notations:
\begin{itemize}
    \item $i^- := \{ 1,2,\ldots,i-1 \}$: the set of indices less than $i$.
    \item $i^+ := \{ i+1, i+2,\ldots,n \}$: the set of indices greater than $i$.
    \item $i^{[+]} := \{ i,i+1,\ldots,n \}$: the set of indices not less than $i$.
    \item $N_{d}[i] := N_{d}(i)\cup \{i\}$: the closed in-degree neighborhood of $i$.
    \item $N_c(S) :=\left(\cup_{i\in S}N_c(i)\right) \setminus S$: the neighbors of all agents in $S$, and $N_{d}(S)$ is defined similarly.
    \item $\cE_c[S_1, S_2]$: the subset of $\cE_c$ containing edges between vertices in the sets $S_1$ and $S_2$.
    \item $\cE_c[S] := \cE_c[S, S]$: the edges within the set $S$.
    \item $\cA_{S}:= \prod_{i\in S}\cA_{i}$: the joint action space of agents in the set $S$.
\end{itemize}

%% file: appendix/Proof_of_Optimality.tex
To facilitate the proof of \cref{thm:optimal}, we first propose four useful lemmas.

\begin{lemma}\label{lem:graph-condition}
    Let $G_c=(\cN, \cE_c)$ be an undirected graph and $G_d=(\cN, \cE_d)$ be a directed graph. 
    If $N_{d}(i) \supseteq N_c(i^{[+]})$ and $j < i$ hold for any $j \in N_{d}(i)$ and $i\in\cN$, then
    \begin{equation}
        \cE_c[k^{[+]}, k^-] = \cE_c[k^{[+]}, N_{d}(k)], \forall k \in \cN.
    \end{equation}
\end{lemma}
\begin{proof}
    We first prove $\cE_c[k^{[+]}, k^-] \supseteq \cE_c[k^{[+]}, N_{d}(k)]$. It suffices to show $N_{d}(k) \subseteq  k^-$, which can be obtained by noting that $j < k$ holds for any $j \in N_{d}(k)$ and $k\in\cN$.

    Next, we prove the reverse inclusion. Assume $i\in k^{[+]}$ and $(i,j) \in \cE_c[k^{[+]}, k^-]$. By definition, $j \in N_c(k^{[+]})$.
    Since $N_{d}(k) \supseteq N_c(k^{[+]})$, it follows that $j \in N_{d}(k)$. Thus, $\cE_c[k^{[+]}, k^-] \subseteq \cE_c[k^{[+]}, N_{d}(k)]$.
\end{proof}

\begin{lemma}\label{lem:nested-neighbor}
    Let $G_c=(\cN, \cE_c)$ be an undirected graph and $G_d=(\cN, \cE_d)$ be a directed graph. 
    If $N_{d}(i) = N_c(i^{[+]})$ and $j < i$ hold for any $j \in N_{d}(i)$ and $i\in\cN$, then
    \begin{equation}
        N_{d}[k] \supseteq N_{d}(k^+), k=1,2,\ldots,n-1.
    \end{equation}
\end{lemma}
\begin{proof}
    We prove the statement by induction.
    
    {\bfseries Base case:} Show that the statement holds for $k=n-1$.
    \begin{equation}
        N_{d}[n-1] = (N_c((n-1)^{[+]})\cup \{n-1\}) \supseteq N_c(n) = N_{d}(n) = N_{d}((n-1)^+).
    \end{equation}

    {\bfseries Induction step:} Assume the statement holds for $k+1$. We prove it for $k$.
    \begin{equation}
        N_{d}[k] = (N_c(k^{[+]})\cup \{k\}) \supseteq N_c(k^+) = N_{d}(k+1).
    \end{equation}
    By the induction hypothesis, we have:
    \begin{equation}
        N_{d}(k+1) = N_{d}[k+1] \setminus \{k+1\} \supseteq N_{d}((k+1)^+) \setminus \{k+1\} .
    \end{equation}
    Therefore,
    \begin{equation}
        N_{d}[k] \supseteq N_{d}(k+1) \cup (N_{d}((k+1)^+) \setminus \{k+1\}) \supseteq N_{d}(k^+).
    \end{equation}

    This completes the proof.
\end{proof}

\begin{lemma}\label{lem:decompose-Q}
    Let $G_c=(\cN, \cE_c)$ be a CG of $Q$, and $G_d=(\cN, \cE_d)$ be a directed graph.
    If $N_{d}(k) \supseteq N_c(k^{[+]})$ and $j < k$ hold for any $j \in N_{d}(k)$ and $k\in\cN$, 
    then given any $i$, there exist functions $Q_1: \cS\times\cA_{N_{d}(i) \cup i^{[+]}}\to\mR$ and $Q_2: \cS\times\cA_{i^-}\to\mR$ such that
    \begin{equation}\label{Q=Q1+Q2}
        Q(s, a) = Q_1(s, a_{i^{[+]}}, a_{N_{d}(i)}) + Q_2(s, a_{i^-}),~~\forall s\in\cS, a\in\cA.
    \end{equation}
\end{lemma}
\begin{proof}
    Decomposing $Q$ according to the structure of the CG $G_c$ yields \begin{equation}
        Q(s, a) = \left( \sum_{(j,k)\in \cE_c[i^{[+]},i^-]}+ \sum_{(j,k)\in \cE_c[i^{[+]}]} + \sum_{(j,k)\in \cE_c[i^-]} \right) Q_{jk}(s,a_{j}, a_{k}).
    \end{equation}
    By \cref{lem:graph-condition}, $\cE_c[i^{[+]}, i^-] = \cE_c[i^{[+]}, N_{d}(i)]$. We rewrite the equation as:
    \begin{equation}
        Q(s, a) = \left( \sum_{(j,k)\in \cE_c[i^{[+]},N_{d}(i)]}+ \sum_{(j,k)\in \cE_c[i^{[+]}]} + \sum_{(j,k)\in \cE_c[N_{d}(i)]} + \sum_{(j,k)\in \cE_c[i^-]\setminus \cE_c[N_{d}(i)]} \right) Q_{jk}(s,a_{j}, a_{k}).
    \end{equation}
    Define:
    \begin{equation}
        Q_1(s, a_{N_{d}(i)}, a_{i^{[+]}}) \triangleq \left( \sum_{(j,k)\in \cE_c[i^{[+]},N_{d}(i)]}+ \sum_{(j,k)\in \cE_c[i^{[+]}]} + \sum_{(j,k)\in \cE_c[N_{d}(i)]} \right)Q_{jk}(s, a_{j}, a_{k}),
    \end{equation}
    which is independent of $a_{i^-\setminus N_{d}(i)}$, and define:
    \begin{equation}
        Q_2(s, a_{i^-}) \triangleq \sum_{(j,k)\in \cE_c[i^-]\setminus \cE_c[N_{d}(i)]}Q_{jk}(s, a_{j}, a_{k}),
    \end{equation}
    which is independent of $a_{i^{[+]}}$.
    Then we have \eqref{Q=Q1+Q2}.
\end{proof}

\begin{lemma}\label{lem:independent-action}
    Let $G_c=(\cN, \cE_c)$ be a CG of $Q$.
    Let $\pi_{-N_{d}[i]}$ be any subpolicy of a joint policy associated with an ADG $G_{d} = (\cN, \cE_{d})$.
    If $N_{d}(k) = N_c(k^{[+]})$ holds for any $k\in\cN$, then we have:
    \begin{equation}
        \underset{\pi_i(\cdot|s,a'_{N_{d}(i)})}{\argmax} \mE_{\pi_{i}, \pi_{-N_{d}[i]}} \left[ Q(s, a'_{N_{d}(i)}, a_{-N_{d}(i)}) \right]
        = \underset{\pi_i(\cdot|s,a'_{N_{d}(i)})}{\argmax} \mE_{\pi_i, \pi_{i^+}}\left[ Q(s, a'_{i^-}, a_{i^{[+]}}) \right], ~~\forall s\in\cS, a'_{i^-}\in\cA_{i^-}.
    \end{equation}
\end{lemma}
\begin{proof}
    By \cref{lem:nested-neighbor}, the actions of the agents in $N_{d}(i^+)$ are a subset of the actions of agents in $N_{d}[i]$. That is, $\pi_{i^+}$ is independent of $a_{i^-\setminus N_{d}(i)}$.
    Therefore, for any $s\in\cS, a'_{i^-}\in\cA_{i^-}$, it holds that
    \begin{equation}\label{eq:independent-action-c}
        \mE_{\pi_{i}, \pi_{-N_{d}[i]}}\left[ Q(s, a'_{i^-}, a_{i^{[+]}}) \right] = \mE_{\pi_i, \pi_{i^+}}\left[ Q(s, a'_{i^-}, a_{i^{[+]}}) \right].
    \end{equation}
    Now, consider:
    \begin{equation}
        \begin{aligned}
            &\underset{\pi_i(\cdot|s,a'_{N_{d}(i)})}{\argmax} \mE_{\pi_{i}, \pi_{-N_{d}[i]}} \left[ Q(s, a'_{N_{d}(i)}, a_{-N_{d}(i)}) \right] \\
            =& \underset{\pi_i(\cdot|s,a'_{N_{d}(i)})}{\argmax} \mE_{\pi_{i}, \pi_{-N_{d}[i]}} \left[ Q_1(s, a'_{N_{d}(i)}, a_{i^{[+]}}) + Q_2(s, a'_{N_{d}(i)}, a_{i^-\setminus N_{d}(i)})\right] & (\text{by \cref{lem:decompose-Q}})\\
            =& \underset{\pi_i(\cdot|s,a'_{N_{d}(i)})}{\argmax} \mE_{\pi_{i}, \pi_{-N_{d}[i]}} \left[ Q_1(s, a'_{N_{d}(i)}, a_{i^{[+]}}) + Q_2(s, a'_{i^-}) \right] & (\text{since } Q_2 \text{ is independent of } a_{i}) \\
            =& \underset{\pi_i(\cdot|s,a'_{N_{d}(i)})}{\argmax} \mE_{\pi_{i}, \pi_{-N_{d}[i]}}\left[ Q(s, a'_{i^-}, a_{i^{[+]}}) \right] \\
            =& \underset{\pi_i(\cdot|s,a'_{N_{d}(i)})}{\argmax} \mE_{\pi_i, \pi_{i^+}}\left[ Q(s, a'_{i^-}, a_{i^{[+]}}) \right]. & (\text{by Equation \eqref{eq:independent-action-c}})
        \end{aligned}
    \end{equation}
    Thus, the proof is complete.
\end{proof}

{\bfseries Proof of \cref{thm:optimal}}

We prove by induction that for every $i\in\cN, a'_{i^-}\in\cA_{i^-}$ and $s\in\cS$:
\begin{equation}
    \mE_{\pi_{i^{[+]}}}\left[ Q^{\pi}(s, a'_{i^-}, a_{i^{[+]}}) \right] = \max_{a_{i^{[+]}}} Q^{\pi}(s, a'_{i^-}, a_{i^{[+]}}).
\end{equation}

{\bfseries Base case:} Show that the statement holds for index $n$.

By the definition of $G_d$-optimality and \cref{lem:independent-action}, for any $s\in\cS, a'_{n^-}\in\cA_{n^-}$:
\begin{equation}
    \begin{aligned}
        \pi_n(\cdot | s, a'_{N_{d}(n)})
        &\in \underset{\pi'_n(\cdot | s, a'_{N_{d}(n)})}{\argmax} \mE_{\pi'_n, \pi_{-N_{d}[n]}}\left[ Q^{\pi}(s, a_{-N_{d}[n]}, a'_{N_{d}(n)}, a_{n}) \right] \\
        &= \underset{\pi'_n(\cdot | s, a'_{N_{d}(n)})}{\argmax} \mE_{\pi'_n}\left[ Q^{\pi}(s, a'_{n^-}, a_{n}) \right].
    \end{aligned}
\end{equation}
This implies:
\begin{equation}
    \mE_{\pi_n}\left[ Q^{\pi}(s, a'_{n^-}, a_{n}) \right] = \max_{\pi'_n(\cdot | s, a'_{N_{d}(n)})} \mE_{\pi'_n}\left[ Q^{\pi}(s, a'_{n^-}, a_{n}) \right] = \max_{a_{n}} Q^{\pi}(s, a'_{n^-}, a_{n}).
\end{equation}

{\bfseries Induction step:} Assume the induction hypothesis holds for $k+1$, we will show that it holds for $k$.

By the definition of $G_d$-optimality and \cref{lem:independent-action}, for any $s\in\cS, a'_{k^-}\in\cA_{k^-}$.
\begin{equation}
    \begin{aligned}
        \pi_k(\cdot | s, a'_{N_{d}(k)}) 
        &\in \underset{\pi'_k(\cdot | s, a'_{N_{d}(k)})}{\argmax} \mE_{\pi'_k, \pi_{-N_{d}[k]}}\left[ Q^{\pi}(s, a'_{N_{d}(k)}, a_{-N_{d}(k)}) \right] \\
        &= \underset{\pi'_k(\cdot | s, a'_{N_{d}(k)})}{\argmax} \mE_{\pi'_k, \pi_{k^+}}\left[ Q^{\pi}(s, a'_{k^-}, a_{k^{[+]}}) \right].
    \end{aligned}
\end{equation}
Thus, we have:
\begin{equation}
    \begin{aligned}
        \mE_{\pi_k, \pi_{k^+}}\left[ Q^{\pi}(s, a'_{k^-}, a_{k^{[+]}}) \right]
        &= \max_{\pi'_k(\cdot | s, a'_{N_{d}(k)})} \mE_{\pi'_k, \pi_{k^+}}\left[ Q^{\pi}(s, a'_{k^-}, a_{k^{[+]}}) \right] \\
        &= \max_{\pi'_k(\cdot | s, a'_{N_{d}(k)})} \mE_{\pi'_k, \pi_{(k+1)^{[+]}}}\left[ Q^{\pi}(s, a'_{k^-}, a_{k^{[+]}}) \right]\\
        &= \max_{\pi'_k(\cdot | s, a'_{N_{d}(k)})} \mE_{\pi'_k}\left[ \max_{a_{(k+1)^{[+]}}} Q^{\pi}(s, a'_{k^-}, a_{k^{[+]}}) \right] & (\text{by induction hypothesis})  \\
        &= \max_{a_{k}}\max_{a_{(k+1)^{[+]}}} Q^{\pi}(s, a'_{k^-}, a_{k^{[+]}}) \\
        &= \max_{a_{k^{[+]}}} Q^{\pi}(s, a'_{k^-}, a_{k}).
    \end{aligned}
\end{equation}

For $k=1$, we have:
\begin{equation}
    V^{\pi}(s) = \mE_{\pi}\left[ Q^{\pi}(s, a) \right] = \max_{a}Q^{\pi}(s, a) = TV^{\pi}(s).
\end{equation}
which shows that $V^{\pi}$ is a fixed point of the Bellman optimal operator $T$. 
Therefore, $V^{\pi}=V^{*}$, implying that $\pi$ is globally optimal. \qed

%% file: appendix/Proof_of_Convergence.tex
Before proving \cref{thm:convergence}, we first propose a theorem showing the convergence relying on a more relaxed condition. Specifically, \cref{thm:convergence-original} only requires that the state-action value functions of the converged policy satisfy the decomposed form of the CG, 
rather than requiring this property for all policies.

\begin{theorem}\label{thm:convergence-original}
    Assume \cref{asp:stopping-criterion} holds. 
    Let $\{\pi^k\}$ be a sequence of joint policies generated by \cref{alg:ad-mpi}, associated with an ADG $G_{d} = (\cN, \cE_{d})$.
    Let $V^{\circ} = \lim_{k\to\infty}V^{\pi^k}$, and let $G_{c} = (\cN, \cE_{c})$ be a CG of state-action value $Q^{V^{\circ}}$.
    If:
    \begin{equation}
        N_{d}(i) = N_c(i^{[+]}), \forall i\in \mathcal{N},
    \end{equation}
    then $\{\pi^k\}$ converges to a globally optimal policy in finite terms.
\end{theorem}

%

The proof of \cref{thm:convergence-original} proceeds in three steps. 
First, we establish the convergence of the joint policy, as shown in \cref{lem:convergence-V}. Next, we prove the convergence of individual agent policies. 
Finally, using \cref{thm:optimal}, we guarantee the global optimality of the converged policy.

\begin{lemma}\label{lem:convergence-V}
    If $\{\pi^k\}$ is a sequence of joint policies generated by \cref{alg:ad-mpi}, associated with an ADG $G_{d} = (\cN, \cE_{d})$,
    then $\{ V^{\pi^k} \}$ will converge to some $V^{\circ}$ in finite terms.
\end{lemma}
\begin{proof}
    According to \cref{alg:ad-mpi}, for any $s\in\cS$ and $k\geq 0$:
    \begin{equation}
        \begin{aligned}
            T_{\pi^{k+1}}V^{\pi^{k}}(s) 
            & = Q^{\pi^{k}}(s, \pi^{k+1}(s)) \\
            & = \max_{a_n}Q^{\pi^{k}}(s, a_{n}, \pi^{k,n}_{-n}(s)) \\
            & \geq Q^{\pi^{k}}(s, \pi^{k,n}(s)) \\
            & = \max_{a_{n-1}}Q^{\pi^{k}}(s, a_{n-1}, \pi^{k,n-1}_{-(n-1)}(s, a_{n-1})) \\
            & \geq Q^{\pi^{k}}(s, \pi^{k,n-1}(s)) \\
            & \cdots \\
            & \geq Q^{\pi^{k}}(s, \pi^{k}(s)) \\
            & = T_{\pi^{k}}V^{\pi^{k}}(s) = V^{\pi^{k}}(s).
        \end{aligned}
    \end{equation}
    Since $T_{\pi^{k+1}}V^{\pi^{k}}(s) \geq V^{\pi^{k}}(s)$, by the monotonicity of the Bellman operator \cite{bertsekas2022abstract}, we have
    \begin{equation}
        V^{\pi^{k+1}}(s) = \lim_{m\rightarrow\infty} T_{\pi^{k+1}}^{m}V^{\pi^{k+1}}(s) \geq T_{\pi^{k+1}}V^{\pi^{k}}(s) \geq V^{\pi^{k}}(s).
    \end{equation}
    Thus, $\{V^{\pi^k}\}$ is non-decreasing and bounded above by the optimal value function $V^{*}$.
    Therefore, $\{V^{\pi^k}\}$ converges to some $V^{\circ}$.
    Further, since the number of deterministic policies is finite, $\{ V^{\pi^k} \}$ will converge to $V^{\circ}$ in finite terms.
\end{proof}

{\bfseries Proof of \cref{thm:convergence-original}}

By \cref{lem:convergence-V}, $\{ V^{\pi^m} \}$ converges to $V^{\circ}$ in finite terms. 
Assume $V^{\pi^m}=V^{\circ}$ for all $m\geq M$.
We prove by induction that for any agent $i\in\cN$ there exist a policy $\pi^{\circ}_i$ and $M_i > M$ such that $\pi^{\circ}_i=\pi^m_i, \forall m \geq M_i$, 
making $\pi^{\circ}_i$ the convergence point. 

{\bfseries Base case:} Show that the statement holds for agent $n$.

For any $s\in\cS, a_{N_{d}(n)}\in\cA_{N_{d}(n)}$, and $m \geq M_n>M$, we have:
\begin{equation}
    \begin{aligned}
        \pi^{m}_n(s, a_{N_{d}(n)})
        &\in \underset{a_{n}\in\cA_{n}}{\argmax} Q^{V^{\pi_{m-1}}}\left(s, \pi^{m-1,n}_{-N_{d}[n]}(s,a_{N_{d}[n]}), a_{N_{d}(n)}, a_{n} \right) \\
        &\stackrel{(a)}{=} \underset{a_{n}\in\cA_{n}}{\argmax} Q^{V^{\circ}}\left(s, \pi^{m}_{n^-\setminus N_{d}(n)}(s,a_{N_{d}(n)}), a_{N_{d}(n)}, a_{n} \right) \\
        &\stackrel{(b)}{=} \underset{a_{n}\in\cA_{n}}{\argmax} \left[ Q_{1}^{V^{\circ}}(s, a_{N_{d}(n)}, a_{n}) + Q_{2}^{V^{\circ}}(s, \pi^{m}_{n^-\setminus N_{d}(n)}(s,a_{N_{d}(n)}), a_{N_{d}(n)}) \right] \\
        &\stackrel{(c)}{=} \underset{a_{n}\in\cA_{n}}{\argmax}Q_{1}^{V^{\circ}}(s, a_{N_{d}(n)}, a_{n}),
    \end{aligned}
\end{equation}
where
\begin{itemize}
    \item $(a)$: $\pi^{m}_{n^-\setminus N_{d}(n)}$ depends at most on actions of agents before $n$ due to the topological sort of $G_d$;
    \item $(b)$: follows from \cref{lem:decompose-Q};
    \item $(c)$: $Q_{2}^{V^{\circ}}$ is independent of $a_{n}$.
\end{itemize}
By \cref{asp:stopping-criterion}, we define $\pi^{\circ}_n := \pi^{M_n}_n = \pi^{m}_n, \forall m \geq M_n$.

{\bfseries Induction step:}
Assume the induction hypothesis holds for all agents $j > k$. We will show that it also holds for agent $k$.

For any $s\in\cS, a_{N_{d}(k)}\in\cA_{N_{d}(k)}$, and $m \geq M_k>\max_{j > k}M_j$, we have:
\begin{equation}
    \begin{aligned}
        \pi^{m}_k(s, a_{N_{d}(k)})
        &\in \underset{a_{k}\in\cA_{k}}{\argmax} Q^{V^{\pi_{m-1}}}\left(s, \pi^{m-1,k}_{-N_{d}[k]}(s,a_{N_{d}[k]}), a_{N_{d}(k)}, a_{k} \right)   \\
        &\stackrel{(a)}{=} \underset{a_{k}\in\cA_{k}}{\argmax} Q^{V^{\circ}}\left(s, \pi^{m}_{k^-\setminus N_{d}(k)}(s, a_{N_{d}(k)}), a_{N_{d}(k)}, a_k, \pi^{m-1}_{k^+}(s,a_{N_{d}[k]})\right) \\
        &\stackrel{(b)}{=} \underset{a_{k}\in\cA_{k}}{\argmax} Q^{V^{\circ}}\left(s, \pi^{m}_{k^-\setminus N_{d}(k)}(s, a_{N_{d}(k)}), a_{N_{d}(k)}, a_k, \pi^{\circ}_{k^+}(s,a_{N_{d}[k]})\right) \\
        &\stackrel{(c)}{=} \underset{a_{k}\in\cA_{k}}{\argmax} \left[ Q_{1}^{V^{\circ}}(s, a_{N_{d}(k)}, a_{k}, \pi^{\circ}_{k^+}(s,a_{N_{d}[k]})) + Q_{2}^{V^{\circ}}(s, a_{N_{d}(k)}, \pi^{m}_{k^-\setminus N_{d}(k)}(s, a_{N_{d}(k)})) \right] \\
        &\stackrel{(d)}{=} \underset{a_{k}\in\cA_{k}}{\argmax}Q_{1}^{V^{\circ}}(s, a_{N_{d}(k)}, a_{k}, \pi^{\circ}_{k^+}(s,a_{N_{d}[k]})).
    \end{aligned}
\end{equation}
The four equlities are explained as follows.
\begin{itemize}
    \item $(a)$: from \cref{lem:nested-neighbor}, $a_{N_{d}(k^+)}$ is a part of $a_{N_{d}[k]}$. Thus, $\pi^{m-1}_{k^+}$ depends at most on the actions of the agents in $N_{d}[k]$. $\pi^{m}_{k^-\setminus N_{d}(k)}$ depends at most on the actions of the agents before $k$ due to the topological sort of $G_d$;
    \item $(b)$: follows from the induction hypothesis;
    \item $(c)$: follows from \cref{lem:decompose-Q};
    \item $(d)$: $Q_{2}^{V^{\circ}}$ is independent of $a_{k}$.
\end{itemize}
By \cref{asp:stopping-criterion}, we define $\pi^{\circ}_k := \pi^{M_k}_k = \pi^{m}_k, \forall m \geq M_k$.

Thus, $\{\pi^m_i\}$ converges to $\pi^{\circ}_i$ in finite terms, and $V^{\circ} = V^{\pi^{\circ}}$.
Further, $\pi^{\circ}_i$ statisfies:
\begin{equation}
    \pi^{\circ}_i(s, a_{N_{d}(i)}) = \underset{a_{i}\in\cA_{i}}{\argmax} Q^{\pi^{\circ}}\left(s, \pi^{\circ}_{-N_{d}[i]}(s,a_{N_{d}[i]}), a_{N_{d}(i)}, a_{i} \right).
\end{equation}
When we view $\pi^{\circ}_i$ as a stochastic policy, it holds that,
\begin{equation}
    \mE_{\pi^{\circ}_{-N_{d}(i)}} \left[ Q^{\pi^{\circ}}(s,a) \right] = \max_{a_{i}} \mE_{\pi^{\circ}_{-N_{d}[i]}} \left[ Q^{\pi^{\circ}}(s,a) \right]
    = \max_{\pi'_i} \mE_{\pi'_i, \pi^{\circ}_{-N_{d}[i]}} \left[ Q^{\pi^{\circ}}(s,a) \right].
\end{equation}
Therefore, $\pi^{\circ}$ is a $G_d$-locally optimal policy.
It follows from \cref{thm:optimal} that $\pi^{\circ}$ is globally optimal. \qed

{\bfseries Proof of \cref{thm:convergence}}

This is a direct consequence of \cref{thm:convergence-original}.

By \cref{lem:convergence-V}, there exist $V^{\circ}$ and $M$ such that $V^{\pi^m} = V^{\circ}$ for all $m\geq M$.
Since $G_{c}$ is a CG of $Q^{\pi}$ for all policies $\pi$, it is also a CG of $Q^{V^{\circ}}=Q^{\pi^M}$.
By \cref{thm:convergence-original}, $\{\pi^k\}$ converges to a globally optimal policy. \qed

%% file: appendix/Proof_of_Preliminary.tex
We characterize a class of Markov games where the equation \eqref{Qsum} holds exactly.
Specifically, we assume the reward function $r$ and transition model $P$ can be decomposed over the edges of $G_c$, ensuring that $G_c$ serves as a CG for the value function $Q^V$.

\begin{proposition}\label{prop:decomposable-game}
    Let $r$ and $P$ denote the reward function and the transition model of a Markov game, respectively. 
    Suppose there exist two sets of functions $\{r_{ij}\}_{(i,j)\in\cE_c}$ and $\{P_{ij}\}_{(i,j)\in\cE_c}$, such that
    \begin{equation}
        r(s,a) = \sum_{(i,j)\in\cE_{c}}r_{ij}(s,a_i,a_j)
    \end{equation}
    and
    \begin{equation}
        P(s'|s,a) = \sum_{(i,j)\in\cE_{c}}P_{ij}(s'|s,a_i,a_j),
    \end{equation}
    where $\cE_{c}$ is the edge set of $G_{c} = (\cN, \cE_{c})$. Then $G_{c}$ is a CG of $Q^{V}$ for any $V\in\cR(\cS)$.
\end{proposition}
\begin{proof}
    For any real-valued function $V: \cS\to\mR$, it holds that
    \begin{equation}
        \begin{aligned}
            Q^{V}(s,a) &= r(s,a) + \gamma\sum_{s'\in\cS}P(s'|s,a)V(s') \\
            &= \sum_{(i,j)\in\cE_{c}}r_{ij}(s,a_i,a_j) + \gamma\sum_{s'\in\cS}\sum_{(i,j)\in\cE_{c}}P_{ij}(s'|s,a_i,a_j)V(s') \\
            &= \sum_{(i,j)\in\cE_{c}}\left( r_{ij}(s,a_i,a_j) + \gamma\sum_{s'\in\cS}P_{ij}(s'|s,a_i,a_j)V(s') \right) \\
            &\triangleq \sum_{(i,j)\in\cE_{c}}Q_{ij}^{V}(s,a_i,a_j).
        \end{aligned}
    \end{equation}
    Therefore, $G_{c}$ is a CG of $Q^{V}$.
\end{proof}

%% file: appendix/Proof_of_Nash.tex
We show the equivalence between the Nash equilibrium defined in \eqref{eq:Nash Equilibrium} and the agent-by-agent optimal defined in \eqref{eq:agent-by-agent} by proposing a proposition as follows.
\begin{proposition}\label{prop:nash}
    Let $\pi$ be a joint policy of independent policies $\pi_i, i\in\cN$. 
    Then, The following claims are equivalent:
    \begin{enumerate}
        \item[(i)] $\mE_{s\sim\mu}\left[ V^{\pi}(s) \right] = \max_{\pi'_i} \mE_{s\sim\mu}\left[ V^{\pi'_i, \pi_{-i}}(s) \right], \mu(s) > 0, \forall s\in\cS ,i\in\cN$.
        \item[(ii)] $V^{\pi}(s) = \max_{\pi'_i} V^{\pi'_i, \pi_{-i}}(s), \forall s\in\cS ,i\in\cN$.
        \item[(iii)] $V^{\pi}(s) = \max_{\pi'_i} \mE_{\pi'_{i}, \pi_{-i}} \left[ Q^{\pi}(s,a) \right], \forall s\in\cS ,i\in\cN$.
    \end{enumerate}
\end{proposition}
\begin{proof}
    (i) $\rightarrow$ (ii). Assume there exist $s\in\cS, i\in\cN$ such that $V^{\pi}(s) < \max_{\pi'_i} V^{\pi'_i, \pi_{-i}}(s)$. 
    Then, let $\pi''$ be a policy such that $V^{\pi''_i, \pi_{-i}}(s)=\max_{\pi'_i} V^{\pi'_i, \pi_{-i}}(s), \forall s\in\cS ,i\in\cN$.
    Thus, $\mE_{s\sim\mu}\left[ V^{\pi}(s) \right] < \mE_{s\sim\mu}\left[ V^{\pi''_i, \pi_{-i}}(s) \right] = \max_{\pi'_i} \mE_{s\sim\mu}\left[ V^{\pi'_i, \pi_{-i}}(s) \right]$, 
    which contradicts (i).

    (ii) $\rightarrow$ (i). This is straightforward by definition.

    (ii) $\rightarrow$ (iii). It is clear that $V^{\pi}(s) \leq \max_{\pi'_i} \mE_{\pi'_{i}, \pi_{-i}} \left[ Q^{\pi}(s,a) \right]$.
    Consider the reverse inequality. For any $\pi'_{i}$, $\lim_{k\to\infty}T^k_{\pi'_i, \pi_{-i}} V^{\pi}(s) = V^{\pi'_i, \pi_{-i}}(s) \leq V^{\pi}(s)$.
    By the monotonicity of the Bellman operator \cite{bertsekas2022abstract}, $\mE_{\pi'_{i}, \pi_{-i}} \left[ Q^{\pi}(s,a) \right] = T_{\pi'_i, \pi_{-i}} V^{\pi}(s) \leq V^{\pi}(s)$.
    Thus, $V^{\pi}(s) \geq \max_{\pi'_i} \mE_{\pi'_{i}, \pi_{-i}} \left[ Q^{\pi}(s,a) \right]$.

    (iii) $\rightarrow$ (ii). Clearly, $V^{\pi}(s) \leq \max_{\pi'_i} V^{\pi'_i, \pi_{-i}}(s)$.
    Consider the reverse inequality. For any $\pi'_{i}$, $T_{\pi'_i, \pi_{-i}} V^{\pi}(s) = \mE_{\pi'_{i}, \pi_{-i}} \left[ Q^{\pi}(s,a) \right] \leq V^{\pi}(s)$.
    By the monotonicity of the Bellman operator, $\lim_{k\to\infty}T^k_{\pi'_i, \pi_{-i}} V^{\pi}(s) = V^{\pi'_i, \pi_{-i}}(s) \leq V^{\pi}(s)$.
    Thus, $V^{\pi}(s) \geq \max_{\pi'_i} V^{\pi'_i, \pi_{-i}}(s)$.
\end{proof}

%% file: appendix/Greedy_Algorithm.tex
To elucidate the construction of ADGs, we present \cref{alg:greedy} that efficiently derives an ADG from a given CG such that the condition \eqref{eq:graph-condition} is satisfied.

\begin{algorithm}[!htb]
    \caption{Greedy Algorithm: Finding a Sparse ADG}\label{alg:greedy}
    \begin{algorithmic}
        \STATE {\bfseries Input:} A CG $G_c$
        \STATE {\bfseries Output:} An ADG $G_d=(\cN,\cE_d)$
        \STATE Initialize an empty graph $G_d=(\cN, \cE_d)$ with all vertices unindexed
        \FOR{$i=0$ {\bfseries to} $n-1$}
        \STATE Assign index $n-i$ to a vertex among the unindexed ones, such that the size of $N_c((n-i)^{[+]})$ is minimized
        \ENDFOR
        \STATE Construct the edge set $\mathcal{E}_d$ by adding edges $(j, i)$ for each vertex $i \in \mathcal{N}$ and each $j \in N_c(i^{[+]})$ as specified in \eqref{eq:graph-condition}
    \end{algorithmic}
\end{algorithm}

%% file: appendix/Experimental_Details.tex



{\bfseries Setup of Coordination Polymatrix Game.}
In matrix cooperative games, different coordination graphs and their corresponding action dependency graphs are shown in \cref{fig:topology_polymatrix}. The payoff matrices in polymatrix coordination games are shown in Tables \ref{tab:Star reward} to \ref{tab:Tree reward}. Any payoff matrix not displayed in the tables is the same as that in \cref{tab:Ring reward1}.
\begin{figure}[!ht]
    \vskip -0.2in
	\begin{center}
		\centerline{\includegraphics[width=\columnwidth]{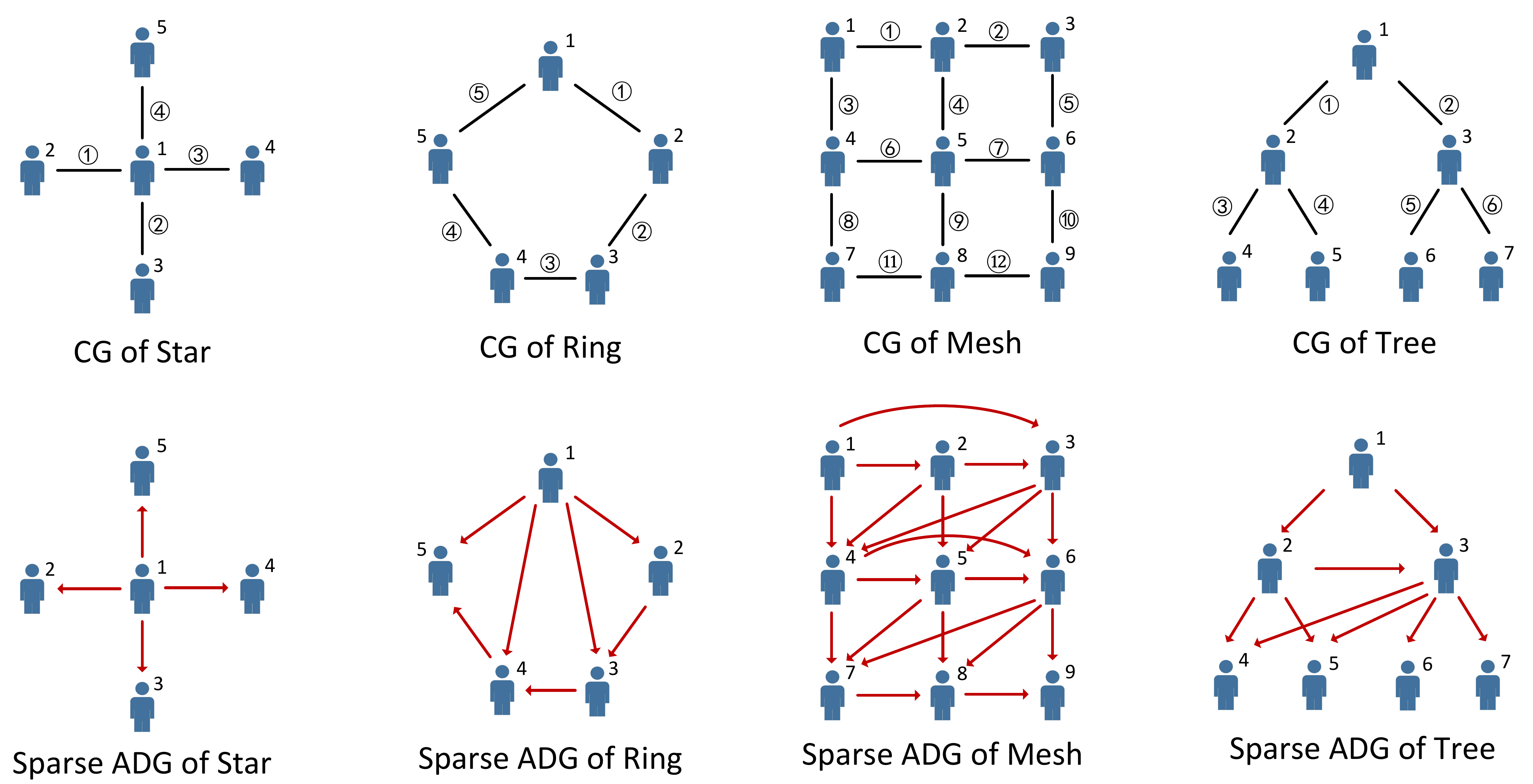}}
		\caption{The CG and sparse ADG of polymatrix coordination game.}
		\label{fig:topology_polymatrix}
	\end{center}
    \vskip -0.2in
\end{figure}

\begin{table}[!ht]
  \centering 
  \begin{minipage}[c]{0.48\textwidth} 
    \centering
    \caption{\textcircled{\footnotesize{1}} in Star}\label{tab:Star reward}
    \vspace{1em}
    \begin{tabular}{cccccc}
      \hline
      \( a_1 \backslash a_2 \) & 0    & 1    & 2    & 3    & 4    \\
      \hline
      0    & 3.5   & \textcolor{red}{5.0}  & 0.5  & 0.5  & 0.5  \\
      \hline
      1    & 0.5   & 3.5  & \textcolor{red}{6.0}  & 0.5  & 0.5  \\
      \hline
      2    & 0.5   & 0.5  & 3.5  & 0.5  & 0.5  \\
      \hline
      3    & 0.5   & 0.5  & 0.5  & 3.25 & 0.5  \\
      \hline
      4    & 0.5   & 0.5  & 0.5  & 0.5  & 3.0  \\
      \hline
    \end{tabular}
  \end{minipage}
  \hfill 
  \begin{minipage}[c]{0.48\textwidth} 
    \centering
    \caption{\textcircled{\footnotesize{2}} in Star} 
    \vspace{1em}
    \begin{tabular}{cccccc}
      \hline
      \( a_1 \backslash a_3 \) & 0    & 1    & 2    & 3    & 4    \\
      \hline
      0    & 3.5   & 0.5  & \textcolor{red}{5.0}  & 0.5  & 0.5  \\
      \hline
      1    & 0.5   & 3.5  & \textcolor{red}{6.0}  & 0.5  & 0.5  \\
      \hline
      2    & 0.5   & 0.5  & 3.5  & 0.5  & 0.5  \\
      \hline
      3    & 0.5   & 0.5  & 0.5  & 3.25 & 0.5  \\
      \hline
      4    & 0.5   & 0.5  & 0.5  & 0.5  & 3.0  \\
      \hline
    \end{tabular}
  \end{minipage}
\end{table}

\begin{table}[!ht]
  \centering 
  \begin{minipage}[c]{0.48\textwidth} 
    \centering
    \caption{\textcircled{\footnotesize{3}} in Star} 
    \vspace{1em}
    \begin{tabular}{cccccc}
    \hline
    \( a_1 \backslash a_4 \)  & 0    & 1    & 2    & 3    & 4    \\
    \hline
    0    & 3.5   & 0.5  & 0.5  & \textcolor{red}{5.0}  & 0.5        \\
    \hline
    1    & 0.5   & 3.5  & \textcolor{red}{6.0}  & 0.5  & 0.5       \\
    \hline
    2    & 0.5   & 0.5  & 3.5  & 0.5  & 0.5       \\
    \hline
    3    & 0.5   & 0.5  & 0.5  & 3.25 & 0.5       \\
    \hline
    4    & 0.5   & 0.5  & 0.5  & 0.5  & 3.0       \\
    \hline
    \end{tabular}
  \end{minipage}
  \hfill 
  \begin{minipage}[c]{0.48\textwidth} 
    \centering
    \caption{\textcircled{\footnotesize{4}} in Star} 
    \vspace{1em}
    \begin{tabular}{cccccc}
    \hline
    \( a_1 \backslash a_5 \)  & 0    & 1    & 2    & 3    & 4    \\
    \hline
    0    & 3.5   & 0.5  & 0.5  & 0.5  & \textcolor{red}{5.0}        \\
    \hline
    1    & 0.5   & \textcolor{blue}{0.5}  & 0.5  & 0.5  & 0.5       \\
    \hline
    2    & 0.5   & 0.5  & 3.5  & 0.5  & 0.5       \\
    \hline
    3    & 0.5   & 0.5  & 0.5  & 3.25 & 0.5       \\
    \hline
    4    & 0.5   & 0.5  & 0.5  & 0.5  & 3.0       \\
    \hline
    \end{tabular}
  \end{minipage}
\end{table}

\begin{table}[!ht]
  \centering 
  \begin{minipage}[c]{0.48\textwidth} 
    \centering
    \caption{\textcircled{1} in Ring}\label{tab:Ring reward1}
    \vspace{1em}
    \begin{tabular}{cccccc}
    \hline
    \( a_1 \backslash a_2 \)  & 0    & 1    & 2    & 3    & 4    \\
    \hline
    0    & 1.0   & 0.1  & 0.1  & 0.1  & 0.1        \\
    \hline
    1    & 0.1   & 1.0  & 0.1  & 0.1  & 0.1       \\
    \hline
    2    & 0.1   & 0.1  & 1.0  & 0.1  & 0.1       \\
    \hline
    3    & 0.1   & 0.1  & 0.1  & 1.0 & 0.1       \\
    \hline
    4    & 0.1   & 0.1  & 0.1  & 0.1  & 1.0       \\
    \hline
    \end{tabular}
  \end{minipage}
  \hfill 
  \begin{minipage}[c]{0.48\textwidth} 
    \centering
    \caption{\textcircled{\footnotesize{2}} in Ring} 
    \vspace{1em}
    \begin{tabular}{cccccc}
    \hline
    \( a_2 \backslash a_3 \)  & 0    & 1    & 2    & 3    & 4    \\
    \hline
    0    & 1.0   & 0.1  & 0.1  & 0.1  & 0.1        \\
    \hline
    1    & 0.1   & 1.0  & \textcolor{red}{2.0}  & 0.1  & 0.1       \\
    \hline
    2    & 0.1   & 0.1  & 1.0  & 0.1  & 0.1       \\
    \hline
    3    & 0.1   & 0.1  & 0.1  & 1.0 & 0.1       \\
    \hline
    4    & 0.1   & 0.1  & 0.1  & 0.1  & 1.0       \\
    \hline
    \end{tabular}
  \end{minipage}
\end{table}

\begin{table}[!ht]
  \centering 
  \begin{minipage}[c]{0.48\textwidth} 
    \centering
    \caption{\textcircled{\footnotesize{3}} in Ring} 
    \vspace{1em}
    \begin{tabular}{cccccc}
    \hline
    \( a_3 \backslash a_4 \)  & 0    & 1    & 2    & 3    & 4    \\
    \hline
    0    & 1.0   & 0.1  & 0.1  & 0.1  & 0.1        \\
    \hline
    1    & 0.1   & 1.0  & 0.1  & 0.1  & 0.1       \\
    \hline
    2    & 0.1   & 0.1  & 1.0  & 0.1  & 0.1       \\
    \hline
    3    & 0.1   & 0.1  & 0.1  & 1.0 & 0.1       \\
    \hline
    4    & 0.1   & 0.1  & 0.1  & 0.1  & 1.0       \\
    \hline
    \end{tabular}
  \end{minipage}
  \hfill 
  \begin{minipage}[c]{0.48\textwidth} 
    \centering
    \caption{\textcircled{\footnotesize{4}} in Ring}
    \vspace{1em}
    \begin{tabular}{cccccc}
    \hline
    \( a_4 \backslash a_5 \)  & 0    & 1    & 2    & 3    & 4    \\
    \hline
    0    & 1.0   & 0.1  & 0.1  & 0.1  & 0.1        \\
    \hline
    1    & 0.1   & 1.0  & \textcolor{red}{2.0}  & 0.1  & 0.1       \\
    \hline
    2    & 0.1   & 0.1  & 1.0  & 0.1  & 0.1       \\
    \hline
    3    & 0.1   & 0.1  & 0.1  & 1.0 & 0.1       \\
    \hline
    4    & 0.1   & 0.1  & 0.1  & 0.1  & 1.0       \\
    \hline
    \end{tabular}
  \end{minipage}
\end{table}

\begin{table}[!ht]
  \centering 
  \begin{minipage}[c]{0.48\textwidth} 
    \centering
    \caption{\textcircled{\footnotesize{5}} in Ring} 
    \vspace{1em}
    \begin{tabular}{cccccc}
    \hline
    \( a_1 \backslash a_5 \)  & 0    & 1    & 2    & 3    & 4    \\
    \hline
    0    & 1.0   & 0.1  & 0.1  & 0.1  & 0.1        \\
    \hline
    1    & 0.1   & \textcolor{red}{10.0}  & 0.1  & 0.1  & 0.1       \\
    \hline
    2    & 0.1   & 0.1  & 1.0  & 0.1  & 0.1       \\
    \hline
    3    & 0.1   & 0.1  & 0.1  & 1.0 & 0.1       \\
    \hline
    4    & 0.1   & 0.1  & 0.1  & 0.1  & 1.0       \\
    \hline
    \end{tabular}
  \end{minipage}
\end{table}

\begin{table}[!ht]
  \centering 
  \begin{minipage}[c]{0.48\textwidth} 
    \centering
    \caption{\textcircled{\footnotesize{7}} in Mesh} 
    \vspace{1em}
    \begin{tabular}{cccccc}
    \hline
    \( a_5 \backslash a_6 \)  & 0    & 1    & 2    & 3    & 4    \\
    \hline
    0    & 1.0   & 0.1  & 0.1  & 0.1  & 0.1        \\
    \hline
    1    & 0.1   & 1.0  & \textcolor{red}{2.0}  & 0.1  & 0.1       \\
    \hline
    2    & 0.1   & 0.1  & 1.0  & 0.1  & 0.1       \\
    \hline
    3    & 0.1   & 0.1  & 0.1  & 1.0 & 0.1       \\
    \hline
    4    & 0.1   & 0.1  & 0.1  & 0.1  & 1.0       \\
    \hline
    \end{tabular}
  \end{minipage}
  \hfill 
  \begin{minipage}[c]{0.48\textwidth} 
    \centering
    \caption{\textcircled{\footnotesize{10}} in Mesh} 
    \vspace{1em}
    \begin{tabular}{cccccc}
    \hline
    \( a_6 \backslash a_9 \)  & 0    & 1    & 2    & 3    & 4    \\
    \hline
    0    & 1.0   & 0.1  & 0.1  & 0.1  & 0.1        \\
    \hline
    1    & 0.1   & \textcolor{red}{10.0}  & 0.1  & 0.1  & 0.1       \\
    \hline
    2    & 0.1   & 0.1  & 1.0  & 0.1  & 0.1       \\
    \hline
    3    & 0.1   & 0.1  & 0.1  & 1.0 & 0.1       \\
    \hline
    4    & 0.1   & 0.1  & 0.1  & 0.1  & 1.0       \\
    \hline
    \end{tabular}
  \end{minipage}
\end{table}

\begin{table}[!ht]
  \centering 
  \begin{minipage}[c]{0.48\textwidth} 
    \centering
    \caption{\textcircled{\footnotesize{2}} in Tree} 
    \vspace{1em}
    \begin{tabular}{cccccc}
    \hline
    \( a_1 \backslash a_3 \)  & 0    & 1    & 2    & 3    & 4    \\
    \hline
    0    & 1.0   & 0.1  & 0.1  & 0.1  & 0.1        \\
    \hline
    1    & 0.1   & \textcolor{red}{2.0}  & \textcolor{red}{1.5}  & 0.1  & 0.1       \\
    \hline
    2    & 0.1   & 0.1  & 1.0  & 0.1  & 0.1       \\
    \hline
    3    & 0.1   & 0.1  & 0.1  & 1.0 & 0.1       \\
    \hline
    4    & 0.1   & 0.1  & 0.1  & 0.1  & 1.0       \\
    \hline
    \end{tabular}
  \end{minipage}
  \hfill 
  \begin{minipage}[c]{0.48\textwidth} 
    \centering
    \caption{\textcircled{\footnotesize{5}} in Tree} 
    \vspace{1em}
    \begin{tabular}{cccccc}
    \hline
    \( a_3 \backslash a_6 \)  & 0    & 1    & 2    & 3    & 4    \\
    \hline
    0    & 1.0   & 0.1  & 0.1  & 0.1  & 0.1        \\
    \hline
    1    & 0.1   & \textcolor{blue}{0.5}  & 0.1  & 0.1  & 0.1       \\
    \hline
    2    & 0.1   & 0.1  & \textcolor{red}{1.5}  & 0.1  & 0.1       \\
    \hline
    3    & 0.1   & 0.1  & 0.1  & 1.0 & 0.1       \\
    \hline
    4    & 0.1   & 0.1  & 0.1  & 0.1  & 1.0       \\
    \hline
    \end{tabular}
  \end{minipage}
\end{table}

\begin{table}[!ht]
  \centering 
  \begin{minipage}[c]{0.48\textwidth} 
    \centering
    \caption{\textcircled{\footnotesize{6}} in Tree}\label{tab:Tree reward}
    \vspace{1em}
    \begin{tabular}{cccccc}
    \hline
    \( a_3 \backslash a_7 \)  & 0    & 1    & 2    & 3    & 4    \\
    \hline
    0    & 1.0   & 0.1  & 0.1  & 0.1  & 0.1        \\
    \hline
    1    & 0.1   & \textcolor{blue}{0.5}  & 0.1  & 0.1  & 0.1       \\
    \hline
    2    & 0.1   & 0.1  & \textcolor{red}{1.5}  & 0.1  & 0.1       \\
    \hline
    3    & 0.1   & 0.1  & 0.1  & 1.0 & 0.1       \\
    \hline
    4    & 0.1   & 0.1  & 0.1  & 0.1  & 1.0       \\
    \hline
    \end{tabular}
  \end{minipage}
\end{table}

{\bfseries Experimental Hyperparameters.}
For algorithms that are not in tabular form, our implementation builds upon the open-source EPyMARL framework \cite{papoudakis2021benchmarking}, employing the Adam optimizer for training.
Unless otherwise specified, we adopt EPyMARL's default hyperparameters, which are also detailed in the config of code provided in supplementary material. 
\cref{tab:hyperparameters} lists the modified hyperparameters for different algorithms and environments. 

In \cref{tab:hyperparameters}, the target update interval is categorized into soft and hard updates: soft updates adjust the target network incrementally during each gradient descent step, while hard updates occur once every specified number of gradient descent steps. 
For the DCG algorithm, we utilize the default hyperparameters from \cite{bohmer2020deep}. 
Although we explored hyperparameter optimization for DCG, we found its performance sensitive to changes, exhibiting instability under alternative configurations.

\begin{table}[!ht]
\centering
\caption{Experimental Hyperparameters}\label{tab:hyperparameters}
\vspace{1em}
\begin{small}
\begin{tabular}{cccccc}
\toprule
 & MAPPO(ATSC) & QMIX(ATSC) & DCG(ATSC) & QMIX(SMAC) & DCG(SMAC) \\
\midrule
learning rate          & 0.0004     & 0.0001     & 0.0005    & $2^{-12}$ & 0.0005\\
weight decay           & 0          & 0          & 0         & $2^{-11}$ & 0 \\
batch size             & 8          & 8          & 32        & 32        & 32 \\
buffer size            & 8          & 32         & 500       & 4096      & 500 \\
target update interval & 0.01(soft) & 0.01(soft) & 100(hard) & 200(hard) & 200(hard) \\
\bottomrule
\end{tabular}
\end{small}
\end{table}

{\bfseries Neural Network Architecture.}
For the MAPPO, QMIX, and DCG algorithms without ADGs, we adopt the default neural network configurations specified in \cite{papoudakis2021benchmarking} and \cite{bohmer2020deep}. 
For MAPPO and QMIX with ADGs, we only modify the agent network architecture as follows. 
Let $o_i \in \mathbb{R}^{d_o}$ denote the observational features for agent $i$, and $a_{N_d(i)} \in \mathbb{R}^{d_a}$ represent the concatenated actions features of neighboring agents in the ADG. The output is computed as:
\begin{equation}
    h_1 = \text{ReLU}(W_1 o_i + b_1), \quad h_2 = \text{ReLU}(W_2 h_1 + b_2), \quad z_o = W_3 h_2 + b_3,
\end{equation}
where $W_1 \in \mathbb{R}^{64 \times d_o}$, $W_2 \in \mathbb{R}^{64 \times 64}$, $W_3 \in \mathbb{R}^{d_{\text{act}} \times 64}$, and $b_1, b_2 \in \mathbb{R}^{64}$, $b_3 \in \mathbb{R}^{d_{\text{act}}}$ are the weights and biases, respectively, mapping to a 64-dimensional vector at each hidden layer and an action-dimensional output $z_o \in \mathbb{R}^{d_{\text{act}}}$. 
Concurrently, the concatenated input $[o_i, a_{N_d(i)}] \in \mathbb{R}^{d_o + d_a}$ is processed through a parallel pipeline with distinct parameters:
\begin{equation}
    h_1' = \text{ReLU}(W_1' [o_i, a_{N_d(i)}] + b_1'), \quad h_2' = \text{ReLU}(W_2' h_1' + b_2'), \quad z_a = W_3' h_2' + b_3',
\end{equation}
where $W_1' \in \mathbb{R}^{64 \times (d_o + d_a)}$, $W_2' \in \mathbb{R}^{64 \times 64}$, $W_3' \in \mathbb{R}^{d_{\text{act}} \times 64}$, and $b_1', b_2' \in \mathbb{R}^{64}$, $b_3' \in \mathbb{R}^{d_{\text{act}}}$. The final output is obtained as:
\begin{equation}
    z = z_o + z_a.
\end{equation}

{\bfseries Setup of ATSC.}
The ATSC environment comprises nine agents. The CG and sparse ADG for ATSC are represented by adjacency lists, as shown in \cref{tab:CG-ATSC} and \cref{tab:ADG-ATSC}, respectively.
Experiments are conducted using ten random seeds (1000, 2000, 3000, 4000, 5000, 6000, 7000, 8000, 9000, 10000), with each run requiring approximately two hours on a system equipped with an Intel Xeon Platinum 8481C CPU @ 2.70GHz, 80GB of memory, and a GeForce RTX 4090 D GPU.
\begin{table}[!ht]
\centering
\begin{minipage}{0.45\textwidth}
\centering
\caption{CG of ATSC}\label{tab:CG-ATSC}
\vspace{1em}
\begin{tabular}{cl}
\hline
\textbf{Vertex} & \textbf{Neighbors} \\
\hline
0 & 1, 3 \\
1 & 0, 2, 4 \\
2 & 1, 5 \\
3 & 0, 4, 6 \\
4 & 1, 3, 5, 7 \\
5 & 2, 4, 8 \\
6 & 3, 7 \\
7 & 4, 6, 8 \\
8 & 5, 7 \\
\hline
\end{tabular}
\end{minipage}
\hfill
\begin{minipage}{0.45\textwidth}
\centering
\caption{Sparse ADG of ATSC}\label{tab:ADG-ATSC}
\vspace{1em}
\begin{tabular}{cl}
\hline
\textbf{Vertex} & \textbf{Out-Degree Neighbors} \\
\hline
0 & \\
1 & 0 \\
2 & 1, 0 \\
3 & 2, 1, 0 \\
4 & 3, 2, 1 \\
5 & 4, 3, 2 \\
6 & 5, 4, 3 \\
7 & 6, 5, 4 \\
8 & 7, 5 \\
\hline
\end{tabular}
\end{minipage}
\end{table}

{\bfseries Setup of SMAC.}
In the MMM2 map of the SMAC environment, our side controls ten agents (seven Marines, two Marauders, and one Medivac), 
while the opposing side commands 12 agents (eight Marines, three Marauders, and one Medivac).
The artificial CG and sparse ADG are represented by adjacency lists, as shown in \cref{tab:CG-SMAC} and \cref{tab:ADG-SMAC}, respectively. 
Experiments utilized ten random seeds (42, 126, 252, 420, 630, 882, 1176, 1512, 1890, 2310), with each run requiring approximately two hours on a system equipped with an Intel Xeon Platinum 8481C CPU @ 2.70GHz, 80GB of memory, and a GeForce RTX 4090 D GPU.
\begin{table}[!ht]
\centering
\begin{minipage}{0.45\textwidth}
\centering
\caption{CG of SMAC}\label{tab:CG-SMAC}
\vspace{1em}
\begin{tabular}{cl}
\hline
\textbf{Vertex} & \textbf{Neighbors} \\
\hline
0 & 1, 3 \\
1 & 0, 2, 4 \\
2 & 1, 5 \\
3 & 0, 4, 6 \\
4 & 1, 3, 5, 7 \\
5 & 2, 4, 8 \\
6 & 3, 7, 9 \\
7 & 4, 6, 8, 9 \\
8 & 5, 7, 9 \\
9 & 6, 7, 8 \\
\hline
\end{tabular}
\end{minipage}
\hfill
\begin{minipage}{0.45\textwidth}
\centering
\caption{Sparse ADG of SMAC}\label{tab:ADG-SMAC}
\vspace{1em}
\begin{tabular}{cl}
\hline
\textbf{Vertex} & \textbf{Out-Degree Neighbors} \\
\hline
0 & \\
1 & 0 \\
2 & 1, 0 \\
3 & 2, 1, 0 \\
4 & 3, 2, 1 \\
5 & 4, 3, 2 \\
6 & 5, 4, 3 \\
7 & 6, 5, 4 \\
8 & 5, 6, 7 \\
9 & 6, 7, 8 \\
\hline
\end{tabular}
\end{minipage}
\end{table}